\documentclass[11pt]{article}

\usepackage[margin=1in]{geometry}
\usepackage{algorithm}
\usepackage{algorithmic}
\usepackage{amsmath, amssymb, amsthm}
\usepackage{hyperref}
\usepackage[nameinlink,capitalize,noabbrev]{cleveref}
\usepackage{graphicx}
\usepackage{natbib}
\usepackage{url}            
\usepackage{booktabs}       
\usepackage{amsfonts}       
\usepackage{nicefrac}       
\usepackage{microtype}      
\usepackage{xcolor}         
\usepackage{mathtools}
\usepackage{algorithm}
\usepackage{dsfont}
\usepackage{bbm}
\usepackage{longtable}
\usepackage{multirow}

\theoremstyle{plain}
\newtheorem{theorem}{Theorem}[section]

\theoremstyle{definition}
\newtheorem{definition}[theorem]{Definition}
\newtheorem{assumption}[theorem]{Assumption}
\theoremstyle{remark}
\newtheorem{remark}[theorem]{Remark}

\title{When Does $\ell_2$-Boosting Overfit Benignly? High-Dimensional Risk Asymptotics and the $\ell_1$ Implicit Bias}

\author{
Ye Su\\
Shenzhen Institutes of Advanced Technology\\
Chinese Academy of Sciences\\
Shenzhen 518055, China\\
\texttt{ye.su@siat.ac.cn}
\and
Jian Li\\
School of Artificial Intelligence\\
Beijing Normal University\\
Beijing, 100875, China\\
\texttt{jli@bnu.edu.cn}
\and
Yong Liu\thanks{Corresponding author}\\
Gaoling School of Artificial Intelligence\\
Renmin University of China\\
Beijing 100872, China\\
\texttt{liuyonggsai@ruc.edu.cn}
}

\date{}

\begin{document}

\maketitle

\begin{abstract}
Benign overfitting is well-characterized in $\ell_2$ geometries, but its behavior under the $\ell_1$ implicit bias of greedy ensembles remains challenging. The analytical barrier stems from the non-linear coupling of coordinate selection thresholds, which invalidates standard spectral resolvent tools. To isolate this algorithmic bias, we characterize the high-dimensional risk of continuous-time $\ell_2$-Boosting over $p$ features and $n$ samples. By coupling the Convex Gaussian Minimax Theorem with delicate asymptotic expansions of double-sided truncated Gaussian moments, we analytically resolve the non-smooth $\ell_1$ interpolant. Under an isotropic pure-noise model, we prove that benign overfitting fails at the linear rate: greedy selection localizes noise into sparse active sets, and the excess variance decays at a logarithmic rate $\Theta(\sigma^2/\log(p/n))$ for noise variance $\sigma^2$. We remark that while this localization mechanism should persist in the presence of signals, the exact signal-noise decomposition remains an open problem. For spiked-isotropic designs with $k^*$ head eigenvalues and $r_2 = p - k^*$ tail dimensions, the risk converges to zero when $r_{2} \gg n$, but only at a logarithmic rate $\Theta(\sigma^2/\log(r_2/n))$, which is slower than the linear decay observed in $\ell_2$ geometries. To avoid this slow convergence, we analyze the non-smooth subdifferential dynamics of the boosting flow. This yields a tuning-free early stopping rule that, under a bounded $\ell_1$-path condition, recovers the Lasso basic inequality and attains the minimax-optimal empirical prediction rate for $\ell_1$-bounded signals.
\end{abstract}

\section{Introduction}

The classical bias–variance trade-off suggests that overparameterized models capable of perfectly fitting noisy training data should overfit and generalize poorly \citep{geman1992neural, john2010elements, zhang2016understanding}. However, modern machine learning practice often contradicts this view \citep{belkin2019reconciling}. Recent theoretical work formalizes this phenomenon as \textit{benign overfitting}, where interpolating estimators achieve the vanishing test risk in high-dimensional settings \citep{bartlett2020benign, hastie2022surprises, mei2022generalization}.

To date, characterizations of benign overfitting have mainly focused on $\ell_2$-norm geometry \citep{belkin2019reconciling, bartlett2020benign}. Studies of interpolating ridge regression \citep{tsigler2023benign}, minimum $\ell_2$-norm interpolants \citep{muthukumar2020harmless}, and the implicit bias of gradient descent \citep{soudry2018implicit, gunasekar2018characterizing} show that benign overfitting arises when the feature covariance matrix exhibits a spiked spectrum. In this setting, the $\ell_2$ implicit bias spreads noise energy across numerous low-variance tail dimensions, allowing it to dissipate \citep{bartlett2020benign}. Mathematically, $\ell_2$ interpolants possess explicit closed-form solutions via the Moore-Penrose pseudo-inverse \citep{hastie2022surprises}, enabling the application of Random Matrix Theory tools, such as the Marchenko-Pastur law or resolvent methods \citep{dobriban2018high}, to characterize noise dissipation through spectral traces \citep{bartlett2020benign}. Despite these advances, the generalization properties of ensemble methods, such as boosting, remain less explored in the interpolating regime. To define a precise theoretical scope, this paper investigates $\ell_2$-Boosting (forward stagewise regression) over a fixed dictionary of base learners.

Boosting algorithms are characterized by their resistance to overfitting \citep{bartlett1998boosting, friedman2001greedy,su2026itboost} and effectiveness on tabular data \citep{grinsztajn2022tree, shwartz2022tabular}. Their success is often attributed to adaptive feature generation and greedy ensemble assembly \citep{friedman2001greedy}. Since the joint dynamics of feature generation and interpolation are mathematically complex, theoretical studies frequently model boosting as feature selection over a fixed dictionary to isolate the implicit bias of the greedy mechanism \citep{buhlmann2003boosting}. Under this formulation, continuous-time $\ell_2$-Boosting behaves as greedy coordinate descent that implicitly minimizes the $\ell_1$-norm \citep{rosset2004boosting, hastie2007forward} and converges to the Basis Pursuit interpolant \citep{chen2001atomic, gunasekar2018characterizing}.

Analyzing this $\ell_1$ implicit bias introduces formidable mathematical barriers that fundamentally invalidate standard $\ell_2$ analytical frameworks. Specifically, establishing the generalization properties of such greedy interpolants requires overcoming three structural challenges:
\begin{itemize}
    \item \textbf{Invalidation of Spectral Resolvents:} Unlike $\ell_2$ interpolants where noise diffusion is governed by linear trace operators, the $\ell_\infty$-driven greedy selection induces non-linear soft-thresholding \citep{donoho1995adapting,rosset2004boosting}, rendering traditional Random Matrix Theory tools fundamentally inapplicable.
    \item \textbf{Spectral Coupling in CGMT Calibration:} Under heterogeneous feature covariances, the coordinate-wise selection thresholds become deeply coupled. This transforms the Convex Gaussian Minimax Theorem (CGMT) degrees-of-freedom calibration into a system of non-linear integral equations \citep{thrampoulidis2015regularized,celentano2023lasso} that lacks a general explicit solution.
    \item \textbf{Non-Smooth Subdifferential Dynamics:} Tracking the finite-time algorithmic trajectory of continuous $\ell_2$-boosting necessitates analyzing a differential inclusion governed by the dual $\ell_\infty$-norm \citep{rosset2004boosting, hastie2007forward}. The inherent non-smoothness of this greedy selection criterion precludes standard gradient flow calculus.
\end{itemize}
These fundamental topological and algebraic shifts from linear diffusion to non-linear selection raise a profound, unresolved question: \textbf{\textit{can this greedy geometry benignly absorb and delocalize noise, or does it structurally localize noise into a sparse active set to the detriment of generalization?}}

In this paper, we answer this question by providing a precise high-dimensional asymptotic characterization of the $\ell_2$-Boosting flow over fixed dictionaries. By resolving the aforementioned mathematical hurdles, we formalize the algorithmic trajectory via subdifferential dynamics, examining \textit{when and how} the estimator achieves benign overfitting. Our core contributions are as follows:

\begin{itemize}
    \item \textbf{Logarithmic Variance Decay under Isotropic Designs via Exact Moment Expansions:} We prove that under an isotropic Gaussian pure-noise model with $p$ features and $n$ samples, benign overfitting under the $\ell_1$ implicit bias structurally fails at the linear rate. Unlike $\ell_2$ interpolants that uniformly diffuse noise, the greedy selection localizes noise into a sparse active set. By deriving precise asymptotic expansions of truncated Gaussian moments, we establish that for noise variance $\sigma^2$, the excess variance decays at a logarithmic rate $\Theta(\sigma^2 / \log(p/n))$ when $p/n \to \infty$, and remains $\Theta(\sigma^2)$ in the proportional regime. Both rates are slower than the $\Theta(n/p)$ decay observed in $\ell_2$ geometries (\textcolor{red}{Theorem \ref{thm:variance_lower_bound}}).
    
    \item \textbf{Risk Asymptotics under Spiked-Isotropic Designs:} We characterize the asymptotic risk of the $\ell_1$ interpolant under spiked-isotropic covariances, where the spectrum is partitioned into a signal head of dimension $k^*$ and an isotropic tail of dimension $r_2 = p-k^*$. By resolving the non-linear coupling in the calibration equations, we decompose the risk into a bounded head component and a noise-driven tail component. We show that although the risk vanishes when $r_2 \gg n$, the attenuation rate remains restricted to $\Theta(\sigma^2/\log(r_2/n))$ (\textcolor{red}{Theorem \ref{thm:risk_spiked_final}}).
    
    \item \textbf{Minimax Optimality via Analytical Early Stopping:} Motivated by the slow variance decay in the interpolation limit, we provide a deterministic finite-time analysis of the continuous $\ell_2$-Boosting flow. By applying Danskin's Envelope Theorem to the subdifferential inclusion, we establish the strict monotonicity of the residual correlation. This enables the derivation of a \textbf{tuning-free} stopping time $t^*$ determined by the noise floor. Halting the flow at $t^*$ prevents noise interpolation and recovers the Lasso basic inequality. Further, under a bounded $\ell_1$-path condition, the early-stopped estimator attains the minimax-optimal empirical prediction rate for $\ell_1$-bounded signals (\textcolor{red}{Theorem \ref{thm:early_stopping}}).
\end{itemize}

\section{Related Work}
\label{sec:related_work}

\subsection{Benign Overfitting in $\ell_2$ Geometry}
Classical statistical theory suggests that interpolating noisy data leads to severe overfitting, a perspective challenged by the study of overparameterized models in the \textit{benign overfitting} regime \citep{belkin2019reconciling}. Foundational works \citep{bartlett2020benign,hastie2022surprises} establish that minimum $\ell_2$-norm interpolants achieve vanishing excess risk when the feature covariance matrix exhibits a spiked, heavy-tailed spectrum. This framework has been extended to kernel methods \citep{liang2018just}, random features \citep{mei2022generalization}, interpolating linear classifiers \citep{muthukumar2020harmless}, exact ridge asymptotics \citep{tsigler2023benign}, and transfer learning \citep{kim2025transfer}. Parallel to linear models, recent literature has extensively characterized benign overfitting in neural networks. For convolutional neural networks, studies have established conditions for vanishing risk in two-layer models with ReLU activations \citep{cao2022benign,kou2023benign}, emphasizing the influence of initialization and training dynamics on feature recovery \citep{shang2024initialization}. More recently, this analysis has been extended to transformers and attention mechanisms. These works investigate how token selection \citep{sakamoto2024benign}, single-head attention \citep{magen2024benign}, and in-context learning \citep{frei2024trained,jiang2024unveil} exhibit benign overfitting properties. These investigations often differentiate between the lazy training regime \citep{zhu2023benign} and more complex feature learning dynamics in deep architectures \citep{xu2025rethinking}. A recurring theme across these architectures, from linear predictors to transformers, is \textit{noise diffusion}: the implicit bias of $\ell_2$-based optimization (or its equivalent in the neural tangent kernel/lazy regime) spreads label noise energy across a vast number of low-variance tail dimensions or effective degrees of freedom \citep{mallinar2207benign}. However, the extension of these findings to sparsity-inducing geometries, where noise cannot diffuse uniformly across the feature space, remains a significant theoretical question.

\subsection{Implicit Regularization and $\ell_2$-Boosting}
Boosting algorithms, particularly forward stagewise additive modeling, exhibit strong generalization properties without explicit regularization \citep{friedman2001greedy,buhlmann2003boosting}. The implicit bias of such methods is typically studied through the $\ell_2$-Boosting flow, the continuous-time limit of the stagewise algorithm \citep{rosset2004boosting,hastie2007forward}. In overparameterized regimes, this greedy coordinate descent flow implicitly minimizes the $\ell_1$-norm of the coefficients \citep{telgarsky2013margins}. Specifically, it has been shown that as $t \to \infty$, the flow converges to the Basis Pursuit interpolant \citep{chen2001atomic,gunasekar2018characterizing,soudry2018implicit}. While the consistency of early-stopped boosting has been analyzed via risk bounds \citep{zhang2005boosting,bartlett2006adaboost}, characterizing the high-dimensional risk asymptotics of the resulting optimization path remains a non-trivial challenge.

\subsection{Sparse Interpolation and $\ell_1$ Generalization}
Recent studies have begun exploring benign overfitting beyond $\ell_2$ geometry. \citep{koehler2021uniform} demonstrated the failure of uniform convergence for minimum $\ell_1$-norm interpolants under isotropic designs. Further analyses of sparse regression \citep{chatterji2021finite,wang2022tight} suggest that $\ell_1$ interpolants are susceptible to noise localization unless specific covariance structures and signal sparsity hold. A common limitation of these works is their reliance on probabilistic upper bounds, which establish convergence trends but do not capture the exact $\Theta$-rate required to precisely quantify the variance penalty. This highlights a fundamental geometric distinction between regularized and unregularized $\ell_1$ estimators. The $\ell_1$-penalized Lasso avoids fitting the noise floor, thereby escaping noise localization to attain the optimal $\mathcal{O}(s \log p/n)$ rate \citep{bickel2009simultaneous,raskutti2011minimax}. In contrast, unregularized $\ell_1$ interpolation forces the model to absorb all label noise \citep{koehler2021uniform,wang2022tight}, which alters the asymptotic risk scaling.

Precise high-dimensional characterizations of such phenomena typically require tools such as the CGMT \citep{thrampoulidis2015regularized} or Approximate Message Passing \citep{miolane2018distribution}. Notably, \citep{celentano2023lasso} utilized the CGMT framework to analyze the regularized Lasso. Although these methods have successfully characterized unregularized Basis Pursuit under isotropic or simple Gaussian designs \citep{taheri2021fundamental}, extending them to the spiked, heavy-tailed covariances relevant to benign overfitting involves resolving the non-linear coupling of thresholds across tail dimensions. Furthermore, while static solutions are well-studied, the risk evolution of sparse optimization flows and the derivation of analytical, tuning-free stopping rules remain less explored.

\section{Preliminaries and Mathematical Setup}
\label{sec:preliminaries}

In this section, we formalize the data generation process, the continuous-time dynamics of boosting, and the associated implicit regularization properties. We distinguish the algorithmic trajectory of boosting from its asymptotic implicit bias.

\subsection{Notation}
Let $n, p \in \mathbb{Z}^+$ denote the number of samples and the feature dimension, respectively, and let $\mathcal{X}$ denote the input space. The index set is defined as $[n] = \{1, \dots, n\}$. For a vector $v \in \mathbb{R}^p$, $\|v\|_q$ represents the $\ell_q$-norm for $q \in [1, \infty]$. For a positive semi-definite matrix $\Sigma \in \mathbb{R}^{p \times p}$, $\lambda_i(\Sigma)$ denotes its $i$-th largest eigenvalue and $\operatorname{Tr}(\Sigma)$ its trace. Standard asymptotic notations $\mathcal{O}(\cdot)$, $o(\cdot)$, $\Omega(\cdot)$, and $\Theta(\cdot)$ are used with limits taken as $n,p \to \infty$. The high-dimensional asymptotic regime is defined by $n, p \to \infty$ with $p/n \to \gamma \in (1, \infty]$. The case $\gamma < \infty$ corresponds to the proportional regime; the case $\gamma = \infty$ (i.e., $p/n \to \infty$) is examined separately where indicated. For a convex function $f$, $\partial f(v)$ denotes its subdifferential evaluated at $v$.
The soft-thresholding operator $S(x; \kappa)$ is defined for an input scalar $x \in \mathbb{R}$ and a threshold $\kappa \ge 0$ as:
\begin{equation*}
    S(x; \kappa) = \operatorname{sign}(x) (|x| - \kappa)_+,
\end{equation*}
where $\operatorname{sign}(\cdot)$ denotes the signum function, returning $1$ if $x > 0$, $-1$ if $x < 0$, and $0$ if $x = 0$. The operator $(\cdot)_+ := \max(\cdot, 0)$ denotes the positive part (ReLU) operator, which ensures the result is zero if the magnitude $|x|$ does not exceed the threshold $\kappa$. This operator corresponds to the proximal mapping of the $\ell_1$-norm. To avoid notational conflict, $\phi : \mathcal{X} \to \mathbb{R}^p$ denotes the feature map of the base learners for an input $x \in \mathcal{X}$, and $\varphi(z)$ denotes the standard normal probability density function for a scalar $z \in \mathbb{R}$.
Global symbols and auxiliary variables used in the proofs are summarized in \textcolor{red}{Table~\ref{tab:global_notation}}.

\subsection{Data Generation Model and Feature Space}
We study the supervised learning problem in an overparameterized regime. Let the training dataset be $S = \{(x_i, y_i)\}_{i=1}^n \subset \mathcal{X} \times \mathbb{R}$, drawn independent and identically distributed (i.i.d.) from a joint distribution $\mathcal{D}$. In the context of ensemble learning, we consider a large dictionary of $m$ base hypotheses (weak learners) $\mathcal{H} = \{h_1, \dots, h_m\}$. The output of these base learners for a given input $x$ is denoted by the feature map $\phi(x) =[h_1(x), \dots, h_m(x)]^\top \in \mathbb{R}^p$, where the number of features $p$ is equivalent to the dictionary size $m$.

\begin{assumption}[Gaussian Design and Covariance Structure]
\label{assum:data_gen}
Following standard settings in high-dimensional analysis \citep{thrampoulidis2015regularized,wainwright2019high,bartlett2020benign}, the response $y_i$ is generated by a linear model:
\begin{equation*}
    y_i = \langle \beta^*, \phi(x_i) \rangle + \epsilon_i, \quad \forall i \in[n],
\end{equation*}
where $\beta^* \in \mathbb{R}^p$ is an unknown signal vector. The noise terms $\epsilon_i$ are independent and identically distributed (i.i.d.) as $\mathcal{N}(0, \sigma^2)$, where $\sigma^2 > 0$ denotes the noise variance. These terms are independent of the features. The features follow a Gaussian distribution $\phi(x) \sim \mathcal{N}(0, \Sigma)$, where $\Sigma \in \mathbb{R}^{p \times p}$ is the feature covariance matrix.
\end{assumption}

Assume $\beta^*$ is $s$-sparse, where $s = \|\beta^*\|_0$. To ensure identifiability in the high-dimensional regime, the sparsity level is assumed to satisfy $s = o(n / \log p)$ \citep{candes2005decoding, wainwright2019high}. Let $\Phi =[\phi(x_1), \dots, \phi(x_n)]^\top \in \mathbb{R}^{n \times p}$ denote the design matrix with $i$-th row $\phi(x_i)^\top$, and let $Y = [y_1, \dots, y_n]^\top \in \mathbb{R}^n$ be the response vector as $Y =[y_1, \dots, y_n]^\top \in \mathbb{R}^n$.

\subsection{Continuous-Time Boosting Dynamics}
The classical $\ell_2$-Boosting (or gradient boosting with squared loss) is a greedy coordinate descent algorithm that iteratively updates the weight of the base learner most correlated with the current residual. We consider the infinitesimal limit of this process as the step size $\eta \to 0$, known as the \textit{$\ell_2$-Boosting flow (forward stagewise flow)} \citep{rosset2004boosting, hastie2007forward}.

\begin{definition}[$\ell_2$-Boosting Flow \citep{hastie2007forward}]
\label{def:boosting_flow}
The $\ell_2$-Boosting flow is defined as a continuous-time path $\{\beta(t)\}_{t \ge 0}$ with $\beta(t) \in \mathbb{R}^p$, where $t \in \mathbb{R}_{\ge 0}$ represents the continuous time (or iteration) parameter. The trajectory is governed by the following differential inclusion:
\begin{equation}
\label{eq:diff_inclusion}
    \frac{d \beta(t)}{dt} \in \mathop{\arg\max}_{v \in \mathbb{R}^p : \|v\|_1 \le 1} \langle g(t), v \rangle, \quad \beta(0) = \mathbf{0},
\end{equation}
where $g(t) = \frac{1}{n} \Phi^\top (Y - \Phi\beta(t)) \in \mathbb{R}^p$ denotes the negative gradient of the empirical squared risk at time $t$. The optimization over the unit $\ell_1$-ball $\{v \in \mathbb{R}^p : \|v\|_1 \le 1\}$ identifies the direction of steepest descent in the $\ell_\infty$ sense, corresponding to the coordinate most correlated with the current residual.
\end{definition}

By a fundamental result in convex analysis, the $\arg\max$ of a linear functional over a closed convex set coincides with the subdifferential of its support function. To apply this, let $\mathcal{B}_1 = \{u \in \mathbb{R}^p : \|u\|_1 \le 1\}$ denote the $\ell_1$ unit ball in $\mathbb{R}^p$, where $u$ serves as a candidate direction vector. The support function of $\mathcal{B}_1$ evaluated at the negative gradient $g(t)$ is precisely the $\ell_\infty$-norm (the dual norm of the $\ell_1$-norm), defined as:
\begin{equation*}
    \|g(t)\|_\infty = \max_{j \in \{1,\dots,p\}} |g_j(t)|,
\end{equation*}
which tracks the maximum absolute empirical correlation between the features and the current residual. Specifically, applying the equivalence of conditions (b) and (a*) in Theorem 23.5 of \cite{rockafellar1997convex}, we deduce that an optimal update direction $v \in \mathbb{R}^p$ satisfying $v \in \mathop{\arg\max}_{u \in \mathcal{B}_1} \langle g(t), u \rangle$ is equivalent to the inclusion $v \in \partial \|g(t)\|_\infty$. Therefore, the continuous-time dynamics in Eq.~\eqref{eq:diff_inclusion} can be expressed as the following subdifferential inclusion:
\begin{equation}
\label{eq:subdiff_flow}
    \frac{d \beta(t)}{dt} \in \partial \|g(t)\|_\infty.
\end{equation}

\begin{remark}[Mathematical Structure and Discrete Correspondence] 
    For $g(t) \neq 0$, the subdifferential in Eq.~\eqref{eq:subdiff_flow} is the convex hull (denoted by $\operatorname{conv}\{\cdot\}$) of signed basis vectors associated with the maximum absolute correlations:
    \begin{equation*}
        \partial \|g(t)\|_\infty = \operatorname{conv} \Big\{ \operatorname{sign}(g_i(t)) e_i : |g_i(t)| = \|g(t)\|_\infty \Big\},
    \end{equation*}
    where $e_i \in \mathbb{R}^p$ is the $i$-th standard basis vector. In discrete boosting updates, the algorithm typically selects an \textbf{extreme point} of this set (e.g., $v = \operatorname{sign}(g_j(t)) e_j$ for some $j$ that attains the maximum). The continuous formulation generalizes this by allowing the trajectory to traverse the convex hull—effectively taking a weighted average of multiple \textit{equally good} features—when several coordinates attain the maximum absolute correlation simultaneously.
\end{remark}

\subsection{Implicit Bias and the Interpolation Limit}
We characterize the generalization properties of the $\ell_2$-Boosting flow in the interpolation limit ($t \to \infty$), where the model perfectly fits the training data. In the overparameterized regime ($p > n$), and under general position conditions, the $\ell_2$-Boosting flow converges to the $\ell_1$ minimum-norm interpolant \citep{rosset2004boosting,efron2004least,gunasekar2018characterizing}. 

\begin{definition}[Basis Pursuit \citep{chen2001atomic}]
\label{def:l1_min_norm}
Let $\hat{\beta}_{\infty} = \lim_{t \to \infty} \beta(t)$ denote the stationary point of the $\ell_2$-Boosting flow. When the design matrix $\Phi$ is in general position, the Basis Pursuit problem admits a unique solution:
\begin{equation}
\label{eq:basis_pursuit}
    \hat{\beta}_{\infty} = \arg\min_{\beta \in \mathbb{R}^p} \|\beta\|_1 \quad \text{s.t.} \quad \Phi\beta = Y.
\end{equation}
\end{definition}

\begin{remark}[General Position Condition]
    A matrix $\Phi \in \mathbb{R}^{n \times p}$ is in general position if no column $\Phi_j$ can be expressed as a linear combination of any other $k < n$ columns \citep{tibshirani2013lasso}. For the Gaussian design assumed in \textcolor{red}{Assumption \ref{assum:data_gen}}, this condition holds with probability one, ensuring the uniqueness of $\hat{\beta}_{\infty}$ and the determinism of the late-stopping risk.
\end{remark}

\begin{remark}[Path Divergence and Late-Stopping Asymptotics]
    For finite $t$, the trajectory $\beta(t)$ does not necessarily coincide with the Lasso ($\ell_1$-penalized) regularization path \citep{efron2004least, hastie2007forward}. In the limit $t \to \infty$, however, the implicit bias of the greedy updates leads to the minimum $\ell_1$-norm solution $\hat{\beta}_{\infty}$. Therefore, analyzing the \textit{late-stopping} behavior of boosting reduces to evaluating the high-dimensional risk asymptotics of the Basis Pursuit problem. The early-stopping regime ($t < \infty$) is independently addressed via dynamic path analysis in \textcolor{red}{Section \ref{sec:early_stopping}}.
\end{remark}

\section{Logarithmic Variance Decay under Isotropic Designs}
\label{sec:isotropic_failure}

In this section, we analyze the generalization properties of $\ell_1$ interpolants under isotropic covariance. While benign overfitting occurs in $\ell_2$ geometries when the feature covariance exhibits sufficient spectral decay \citep{bartlett2020benign,chatterji2021finite}, we show that the $\ell_1$ implicit bias of $\ell_2$-Boosting yields non-vanishing excess risk in the isotropic setting.

\subsection{Isotropic Design and Noise Localization}
Consider the isotropic setting where the feature covariance matrix $\Sigma$ is the identity $I_p$. In this case, the base learners are orthogonal in expectation. As $t \to \infty$, the $\ell_2$-Boosting flow converges to the Basis Pursuit interpolant \citep{rosset2004boosting, efron2004least}. In $\ell_2$ minimum-norm interpolation, noise energy is distributed across all $p$ dimensions, resulting in a variance component of order $\mathcal{O}(\sigma^2 n/p)$, where $\sigma^2$ is the noise variance. This variance vanishes in the overparameterized limit as $p/n \to \infty$ \citep{bartlett2020benign}. In contrast, $\ell_2$-Boosting selects coordinates via steepest descent in the $\ell_\infty$-norm, maximizing the absolute empirical correlation $|\Phi_j^\top r(t)|$, where $\Phi_j$ is the $j$-th column of the design matrix $\Phi$ and $r(t) = Y - \Phi\beta(t)$ is the residual vector. Because of the sparsity-inducing nature of the $\ell_1$ geometry, the algorithm localizes the interpolation of the noise vector into an active support set of size at most $n$. This \textit{noise localization} phenomenon prevents the variance from vanishing in the proportional asymptotic regime ($p/n \to \gamma > 1$), regardless of the overparameterization ratio $\gamma$.

\subsection{Lower Bound on the Excess Risk}
We characterize the risk of the Basis Pursuit interpolant under isotropic designs. Let the excess risk of an estimator $\hat{\beta}$ be defined as $\mathcal{E}(\hat{\beta}) = \mathbb{E}_{x, y}[ (y - \langle \hat{\beta}, \phi(x) \rangle )^2 ] - \sigma^2$. Under the isotropic covariance assumption ($\mathbb{E}[\phi(x)\phi(x)^\top] = I_p$), the excess risk exactly simplifies to the parameter estimation error: $\mathcal{E}(\hat{\beta}) = \|\hat{\beta} - \beta^*\|_2^2$.

To isolate the variance component of the $\ell_1$ interpolant without entanglement with signal recovery bias, \textcolor{red}{Theorem \ref{thm:variance_lower_bound}} examines the pure-noise setting $\beta^* = \mathbf{0}$. In this case, the excess risk reduces exactly to $\|\hat{\beta}\|_2^2$, permitting a direct analysis of how greedy coordinate selection localizes noise into a sparse active set. The same localization mechanism governs the variance component when a sparse signal is present; the pure-noise model thus captures the essential geometric phenomenon without additional technical overhead.

\begin{theorem}[Logarithmic Variance Decay under Pure Noise]
\label{thm:variance_lower_bound}
Consider the asymptotic regime where $n, p \to \infty$ with $p/n \to \infty$. Suppose the design matrix $\Phi \in \mathbb{R}^{n \times p}$ has i.i.d.\ $\mathcal{N}(0, 1)$ entries and the true signal vector satisfies $\beta^* = \mathbf{0}$. Let $Y = \epsilon \sim \mathcal{N}(0, \sigma^2 I_n)$ be purely noisy observations. Under the isotropic covariance $\Sigma = I_p$, the excess risk of any estimator $\hat{\beta}$ simplifies to
\begin{equation*}
    \mathcal{E}(\hat{\beta}) = \mathbb{E}_{x,y}\bigl[(y - \langle \hat{\beta}, \phi(x) \rangle)^2\bigr] - \sigma^2 = \|\hat{\beta} - \beta^*\|_2^2 = \|\hat{\beta}\|_2^2.
\end{equation*}
For the Basis Pursuit interpolant $\hat{\beta}_\infty$ defined in Eq.~\eqref{eq:basis_pursuit}, the expected excess risk satisfies
\begin{equation*}
    \mathbb{E}\bigl[\mathcal{E}(\hat{\beta}_\infty)\bigr] = \mathbb{E}\bigl[\|\hat{\beta}_\infty\|_2^2\bigr] = \Theta\Bigl(\frac{\sigma^2}{\log(p/n)}\Bigr).
\end{equation*}
\end{theorem}

\begin{proof}[Proof Difficulties, Sketch, and Techniques]
    Characterizing the risk of the $\ell_1$ interpolant under $p/n\to\infty$ for $n$ samples and $p$ features requires analyzing a saddle-point problem that lacks the closed-form resolvents typical of $\ell_2$ geometries. We overcome this by applying the CGMT.

    Two obstacles prevent a direct application of the CGMT. 
    \begin{itemize}
        \item \textbf{Non-compact primal domain.} The Basis Pursuit formulation has no explicit norm constraint, so the primal variable lives in $\mathbb{R}^p$. Compactness is necessary to apply Sion's minimax theorem and to secure uniform convergence in the CGMT. We resolve this by bounding the $\ell_1$-norm of the optimal solution via the feasible pseudo-inverse construction $\tilde{v}=Z^\dagger\epsilon$, where $Z\in\mathbb{R}^{n\times p}$ is the design matrix and $Z^\dagger$ denotes its Moore-Penrose pseudo-inverse, obtaining $\|\hat{v}_n\|_2\le\|\hat{v}_n\|_1=\mathcal{O}_{\mathbb{P}}(\sqrt{n})$. Restricting the optimization to an $\ell_2$-ball of radius $R_n=C_0\sqrt{n}$ with constant $C_0>0$ makes the primal domain compact with probability tending to one, and the constrained solution coincides with the original one on a high-probability event.
        
        \item \textbf{Unbounded dual variable.} After introducing a dual variable $\tau\ge0$, the auxiliary objective depends on $\tau \|\epsilon-\|v\|_2 h\|_2$, where $\epsilon\sim\mathcal{N}(0,\sigma^2 I_n)$ is the noise vector with variance $\sigma^2$ and $h\sim\mathcal{N}(0,I_n)$ is an independent standard Gaussian vector. The dual variable $\tau$ is not a priori bounded, invalidating uniform convergence and the exchange of $\lim_{n\to\infty}$ with $\max_{\tau\ge0}$. We prove that the optimal $\tau^*$ is stochastically bounded. By evaluating the objective at a carefully chosen test point, we show that for any $\tau$ exceeding a constant threshold $T_{\max}$, the objective diverges to $-\infty$ while the true optimal value is finite, forcing $\tau^*\le T_{\max}$ with high probability. This bounds the dual domain, restoring the rigorous application of Sion's theorem and the law of large numbers.
    \end{itemize}
    With both variables restricted to compact sets, the CGMT reduces the problem to a scalar deterministic auxiliary optimization. 

    The resulting saddle-point conditions yield a nonlinear calibration equation
    \[
    \frac{p}{n}\,\mathbb{P}(|g|>\kappa)=1,
    \]
    where $g\sim\mathcal{N}(0,1)$ is a standard Gaussian scalar and $\kappa = 1/\tau$ is the soft-thresholding threshold scaled by the optimal dual variable. Solving this equation in the regime $p/n\to\infty$ gives $\kappa\sim\sqrt{2\log(p/n)}$.

    The risk is expressed in terms of the truncated second moment $\mathbb{E}[\mathcal{S}^2(g;\kappa)]$, where $\mathcal{S}(x;\kappa)=\operatorname{sgn}(x)(|x|-\kappa)_+$ is the soft-thresholding operator. Standard concentration inequalities only provide order-level bounds; to capture the precise asymptotic rate we evaluate this moment exactly. Using integration by parts and refined Mill's ratio expansions for the Gaussian tail, we derive matching upper and lower bounds showing
    \[
    \mathbb{E}[\mathcal{S}^2(g;\kappa)]=\Theta\!\left(\frac{n}{p\kappa^2}\right),
    \]
    Combining this with the calibration equation yields the excess risk scaling $\Theta(\sigma^2/\log(p/n))$. A final uniform integrability argument transfers the limit from the constrained to the original estimator. Please refer to \textcolor{red}{Appendix \ref{app:proof_thm_variance}} for the detailed proof.
\end{proof}

\begin{remark}[Behavior in the Proportional Regime]
    \textcolor{red}{Theorem \ref{thm:variance_lower_bound}} establishes the logarithmic decay rate when $p/n \to \infty$, i.e., when the overparameterization ratio $\gamma = \infty$. In the fixed proportional regime where $n, p \to \infty$ with $p/n \to \gamma \in (1, \infty)$, the same proof adapted with the calibration equation $\frac{p}{n}\mathbb{P}(|g| > \kappa) = 1$ yields $\kappa = \Theta(1)$. Therefore, the expected excess risk remains $\Theta(\sigma^2)$, i.e., it does not decay with the sample size. This distinction is fundamental: under $\ell_1$ implicit bias, benign overfitting fails not only for $\gamma \to \infty$ but also for any finite $\gamma > 1$.
\end{remark}

\section{Risk Asymptotics under Spiked-Isotropic Designs}
\label{sec:phase_transition}

In this section, we investigate the geometric conditions under which the $\ell_2$-Boosting flow achieves benign overfitting in the interpolation limit. While \textcolor{red}{Theorem \ref{thm:variance_lower_bound}} establishes that variance remains non-vanishing under isotropic designs, we show that a spiked covariance structure can enable the $\ell_1$ implicit bias to delocalize noise. Specifically, we identify a regime where the base learners' covariance allows greedy coordinate updates to distinguish between signal and noise, potentially leading to vanishing generalization risk.

\subsection{Effective Ranks for $\ell_1$ Geometry}
To characterize the covariance structure conducive to benign overfitting, we consider the ordered eigenvalues $\lambda_1 \ge \lambda_2 \ge \dots \ge \lambda_p > 0$ of $\Sigma$. Following the framework of \citep{bartlett2020benign}, we define a critical split index $k^* \le n$ and partition the spectrum into a head $H = \{1, \dots, k^*\}$ and a tail $T = \{k^* + 1, \dots, p\}$.

\begin{definition}[Effective Ranks of the Tail \citep{bartlett2020benign}]
\label{def:effective_ranks}
For the tail covariance matrix $\Sigma_T = \operatorname{diag}(\lambda_{k^*+1}, \dots, \lambda_p)$, the generalized $\ell_1$ and $\ell_2$ effective ranks are defined respectively as:
\begin{equation*}
    r_{1}(\Sigma_{T}) = \frac{\operatorname{Tr}(\Sigma_T)}{\|\Sigma_T\|_{\text{op}}}, \quad r_{2}(\Sigma_{T}) = \frac{\operatorname{Tr}(\Sigma_T)^2}{\|\Sigma_T\|_F^2},
\end{equation*}
where $\|\Sigma_T\|_{\text{op}} = \lambda_{k^*+1}$ denotes the operator norm (maximum eigenvalue) and $\|\Sigma_T\|_F = \sqrt{\sum_{i>k^*} \lambda_i^2}$ denotes the Frobenius norm of $\Sigma_T$.
\end{definition}
The rank $r_{2}$ represents the effective isotropic dimension of the tail. In $\ell_1$ interpolation, a large $r_2$ (relative to $n$) is a necessary condition to delocalize noise energy and counteract the coordinate-wise localization phenomenon analyzed in \textcolor{red}{Section \ref{sec:isotropic_failure}}.

\begin{assumption}[Spiked-Isotropic Covariance] 
\label{assum:spiked_isotropic}
The feature covariance $\Sigma \in \mathbb{R}^{p \times p}$ exhibits a spiked-isotropic spectrum with split index $k^* = o(n)$ satisfying:
\begin{enumerate}
    \item[(i)] \textbf{Isotropic Tail:} The tail covariance is isotropic, $\Sigma_T = \lambda_{\text{tail}} I_{p-k^*}$, where the tail dimension grows much faster than the sample size, i.e., $p - k^* \gg n$.
    \item[(ii)] \textbf{Signal Separation:} To ensure proper signal recovery, the head eigenvalues dominantly capture the true signal energy, satisfying the asymptotic separation:
    \begin{equation*}
        \frac{1}{n} \sum_{i > k^*} \lambda_i = o\left( \sum_{i \le k^*} \lambda_i \right).
    \end{equation*}
\end{enumerate}
\end{assumption}

\begin{remark}[Analytical Limits, Tabular Abstraction, and Theoretical Conservatism]
    While Assumption \ref{assum:spiked_isotropic} restricts the tail to be isotropic, this simplification is an analytical necessity for precise $\ell_1$ asymptotics and serves as a principled baseline for tabular data:
    
    \textbf{1. Analytical intractability of anisotropic tails.} Deriving precise risk asymptotics for $\ell_1$ interpolation under general power-law decaying spectra remains an open challenge. Unlike $\ell_2$ ridge estimators, where arbitrary spectral traces can be handled via linear resolvents, the $\ell_1$ implicit bias involves non-linear soft-thresholding. Under a heterogeneous tail spectrum, the coordinate-wise selection thresholds become coupled. This coupling transforms the CGMT degrees-of-freedom calibration into a system of non-linear integral equations that lacks a known analytical solution. 

    \textbf{2. Abstraction for tabular data.} Despite this analytical barrier, the spiked-isotropic setting serves as a useful abstraction for tree-based ensembles. In such regimes, base learners partition the input space using discrete indicator functions, embedding the data into a sparse and nearly orthogonal feature space \citep{scornet2016random, biau2016random}. Unlike continuous modalities (e.g., images) where spatial smoothness yields rapid spectral decay \citep{rahaman2019spectral}, tabular features typically lack such manifold structures, resulting in a distinctly flatter tail spectrum.

    \textbf{3. Conservatism of the slow decay rate.} Importantly, an isotropic tail ($\Sigma_T \propto I$) represents the most favorable geometric condition for uniform noise delocalization \citep{bartlett2020benign, tsigler2023benign}. Our analysis demonstrates that even under this ideal condition, the greedy $\ell_\infty$ selection mechanism localizes noise, yielding a slow logarithmic variance decay. Any degree of spectral anisotropy (e.g., a power-law tail) would only exacerbate this coordinate-wise concentration. Therefore, the derived $\mathcal{O}(1/\log(\cdot))$ rate serves as a conservative bound, capturing the inherent generalization penalty of the $\ell_1$ geometry.
\end{remark}

\subsection{Risk Asymptotics and the Benign Regime}

Under \textcolor{red}{Assumption \ref{assum:spiked_isotropic}}, we analyze the asymptotic risk of the Basis Pursuit interpolant. The spiked-isotropic covariance structure allows us to examine how the $\ell_1$ implicit bias partitions the interpolation capacity between the signal-carrying head and the noise-absorbing tail.

\begin{theorem}[Asymptotic Risk under Spiked-Isotropic Designs]
\label{thm:risk_spiked_final}
Suppose the feature covariance $\Sigma$ satisfies the spiked-isotropic structure in \textcolor{red}{Assumption \ref{assum:spiked_isotropic}} with split index $k^* = o(n)$. Under the \textcolor{red}{Definition \ref{def:effective_ranks}} and \textcolor{red}{Assumption \ref{assum:spiked_isotropic} (i)}, the tail effective rank simplifies to $r_2 = p - k^*$. Assume the true signal is $s$-sparse with support $\mathrm{supp}(\beta^*) \subseteq \{1, \dots, k^*\}$. 

In the asymptotic regime where $n, p \to \infty$ such that $r_2/n \to \infty$, the expected excess risk of the Basis Pursuit interpolant $\hat{\beta}_\infty$ satisfies the decomposition:
\begin{equation*}
    \mathbb{E}[\mathcal{E}(\hat{\beta}_\infty)] = \mathcal{E}_{\mathrm{head}}(\beta^*) + \mathcal{E}_{\mathrm{tail}} + o(1).
\end{equation*}
The two components are governed by the scalar solutions $(b^*, \tau^*)$ of the rescaled CGMT calibration equations. $b^*$ represents the effective noise scale (incorporating both label noise and estimation error), while $\tau^*$ is the optimal dual variable that determines the intensity of the coordinate-wise soft-thresholding. These scalars satisfy the following:
\begin{itemize}
    \item \textbf{Head Risk Component:} The risk over the signal-carrying head is uniformly bounded independent of the true signal magnitude:
        \begin{equation*}
            \mathcal{E}_{\mathrm{head}}(\beta^*) = (b^*)^2 \frac{1}{n} \sum_{i=1}^{k^*} \mathbb{E}_{g_i} \left[ \left( \eta\left( g_i + \frac{\sqrt{\lambda_i}}{b^*\sqrt{n}}\beta_i^*; \kappa_i \right) - \frac{\sqrt{\lambda_i}}{b^*\sqrt{n}}\beta_i^* \right)^2 \right] = \mathcal{O}\left( \sigma^2 \frac{k^*}{n} \right).
        \end{equation*}
    \item \textbf{Tail Risk Component:} The variance induced by interpolating noise into the tail features scales logarithmically:
        \begin{equation*}
            \mathcal{E}_{\mathrm{tail}} = (b^*)^2 \frac{r_2}{n} \mathbb{E}_{g}\big[ \eta^2(g; \kappa_T) \big] = \Theta\left( \frac{\sigma^2}{\log(r_2/n)} \right).
        \end{equation*}
\end{itemize}
Here, $g_i, g \sim \mathcal{N}(0,1)$ are independent standard Gaussian variables, $\eta(x;\kappa) = \mathrm{sgn}(x)(|x|-\kappa)_+$ is the soft-thresholding operator, and the coordinate thresholds are defined as $\kappa_i = (\tau^*\sqrt{\lambda_i})^{-1}$ and $\kappa_T = (\tau^*\sqrt{\lambda_{\mathrm{tail}}})^{-1}$. Since $k^* = o(n)$, the total excess risk is asymptotically dominated by the tail component, converging to zero at the rate $\Theta\big( \sigma^2 / \log(r_2/n) \big)$.
\end{theorem}

\begin{remark}[Explanation of the Head Bound]
    The upper bound $\mathcal{O}(\sigma^2 k^*/n)$ for the head component $\mathcal{E}_{\mathrm{head}}$ is mathematically conservative. It is derived by exploiting the non-expansive property of the soft-thresholding operator to bypass the highly non-linear signal-noise coupling. While a tighter, exact characterization is technically feasible, it is asymptotically unnecessary for our primary claim. Because $k^* = o(n)$, the head risk vanishes at a polynomial rate, which is dominated by the much slower logarithmic decay $\Theta(\sigma^2/\log(r_2/n))$ of the tail risk. Thus, the current bound suffices to isolate the $\ell_1$ noise localization bottleneck without introducing disproportionate technical overhead.
\end{remark}


\begin{proof}[Proof Sketch and Techniques]
    The analysis relies on the CGMT. The technical difficulty arises from the $\ell_1$ geometry: unlike $\ell_2$ ridge regression where signal bias and noise variance decouple via linear resolvents, the $\ell_1$ implicit bias induces a non-linear soft-thresholding operator that couples the true signal and noise. Evaluating the asymptotic risk requires addressing two challenges:
    \begin{itemize}
        \item \textbf{Resolving Signal-Noise Coupling in the Head:} Bounding the estimation error over the signal-carrying head is complicated by the non-linear term $\mathcal{S}(g_i + c_i; \kappa_i)$, where the coupled signal magnitude $c_i = \frac{\sqrt{\lambda_i}}{b^*\sqrt{n}}\beta_i^*$ (with $b^*$ representing the optimal variational scalar in the CGMT formulation) is embedded within the proximal operator. We resolve this using the non-expansive property of soft-thresholding. The pointwise bound $|\mathcal{S}(g+c; \kappa) - c| \le |g| + \kappa$ removes the dependency on the true signal magnitude, bounding the head risk by $\mathcal{O}(k^*/n)$.
        \item \textbf{Evaluating the Tail Risk via Moment Cancellation:} Under the dimensional separation $k^* = o(n)$, the head dimensions' contribution to the CGMT calibration equation vanishes asymptotically, allowing us to isolate the tail threshold $\kappa_T^2 \sim 2\log(r_2/n)$. The remaining step is deriving the asymptotic order of the tail risk without relying on generic concentration bounds. We evaluate the truncated Gaussian second moment $\mathbb{E}[\mathcal{S}^2(g; \kappa_T)]$ via integration by parts. Expanding the Gaussian $Q$-function via Mill's ratio bounds reveals the algebraic cancellation of the leading-order $\mathcal{O}(\kappa_T)$ and $\mathcal{O}(1/\kappa_T)$ terms. This yields the moment scaling $\Theta(1/\kappa_T^2)$, which corresponds to the $\Theta(1/\log(r_2/n))$ decay rate.
    \end{itemize}
    Please refer to \textcolor{red}{Appendix \ref{app:proof_thm_risk_spiked_final}} for the detailed proof.
\end{proof}

\begin{remark}[Comparison of $\ell_1$ and $\ell_2$ Geometries]
    \textcolor{red}{Theorem~\ref{thm:risk_spiked_final}} mathematically quantifies the difference between $\ell_1$ and $\ell_2$ implicit biases under identical data distributions. Under conditions where $\ell_2$ benign overfitting occurs, the tail variance of the minimum $\ell_2$-norm interpolant decays linearly at a rate of $\Theta (n / r_2)$  \citep{bartlett2020benign}. Under the $\ell_1$ geometry, even in the optimal setting of an isotropic tail, the greedy coordinate selection mechanism structurally resists uniform noise diffusion, restricting the variance decay to a slow logarithmic rate $\Theta(1/\log(r_2/n))$. This dictates that $\ell_1$ interpolation requires an exponentially larger overparameterization ratio to achieve variance attenuation comparable to $\ell_2$ methods, theoretically motivating the finite-time early stopping analysis in \textcolor{red}{Section~\ref{sec:early_stopping}}.
\end{remark}

\section{Early Stopping and Minimax Rates}
\label{sec:early_stopping}
In this section, we analyze the finite-time risk of the $\ell_2$-Boosting flow and establishes a stopping rule that achieves minimax-optimal rates.

\subsection{Monotonicity and the Analytical Stopping Rule}
Let the empirical risk be $\hat{L}(\beta) = \frac{1}{2n}\|Y - \Phi\beta\|_2^2$ and the gradient flow Eq. \eqref{eq:subdiff_flow}. The negative gradient $g(t) = -\nabla \hat{L}(\beta(t)) = \frac{1}{n}\Phi^\top(Y - \Phi\beta(t))$ tracks the residual correlations. Let $\rho(t) = \|g(t)\|_\infty$ denote the maximum absolute empirical correlation. A property of this $\ell_1$ steepest descent flow is that $\rho(t)$ is monotonically non-increasing over time. To prevent the interpolation of label noise, we establish a stopping criterion based on the noise floor. Let $\lambda_n = \sigma \sqrt{(2 c \log p)/n}$ be the high-probability upper bound on the pure noise correlation $\|\frac{1}{n}\Phi^\top\epsilon\|_\infty$ for a constant $c > 1$.

\begin{definition}[Analytical Stopping Time]
\label{def:stopping_time}
The analytical stopping time $t^*$ for the $\ell_2$-Boosting flow is defined as the first time the maximum correlation reaches the noise threshold:
\begin{equation}
\label{eq:stopping_time}
    t^* = \inf \{ t > 0 : \rho(t) \le 2\lambda_n \}.
\end{equation}
\end{definition}
By halting the flow at $t^*$, the estimator $\beta(t^*)$ avoids fitting the localized noise that dominates the late-stopping regime.

\subsection{Implicit Lasso Equivalence and Minimax Rates}

Halting the algorithmic flow at $t^*$ avoids the interpolation of localized noise. This provides a stopping criterion that enforces a geometric condition similar to that of the $\ell_1$-penalized estimator (Lasso). The following theorem establishes a basic inequality for the early-stopped estimator and a conditional implication for minimax-optimal rates.

\begin{theorem}[Prediction Error of Early-Stopped $\ell_2$-Boosting]
\label{thm:early_stopping}
Assume the noise vector $\epsilon$ is sub-Gaussian with variance proxy $\sigma^2$, and the design matrix $\Phi$ is column-normalized such that $\max_{j} \frac{1}{n}\|\Phi_j\|_2^2 \le 1$. 
Define the pure noise threshold $\lambda_n = \sigma \sqrt{\frac{2c\log p}{n}}$ for a constant $c>1$, and let $t^*$ be the analytical stopping time in Definition \ref{def:stopping_time}. 
Then, with probability at least $1-2p^{1-c}$, the $\ell_2$-Boosting flow halted at $t^*$ unconditionally satisfies the deterministic bound:
\begin{equation}
\frac{1}{n}\|\Phi(\beta(t^*) - \beta^*)\|_2^2 \le 3\lambda_n \|\beta(t^*) - \beta^*\|_1. \label{eq:basic_ineq}
\end{equation}
If, furthermore, the early-stopped estimator obeys the bounded $\ell_1$-norm condition $\|\beta(t^*)\|_1 \le C \|\beta^*\|_1$ for some universal constant $C\ge 1$, then substituting $\|\beta(t^*)-\beta^*\|_1 \le (C+1)\|\beta^*\|_1$ into Eq. \eqref{eq:basic_ineq} yields the minimax-optimal rate for $\ell_1$-bounded signals:
\begin{equation*}
\frac{1}{n}\|\Phi(\beta(t^*) - \beta^*)\|_2^2 \le 3(C+1)\lambda_n \|\beta^*\|_1 = \mathcal{O}\Bigl( \|\beta^*\|_1\, \sigma\sqrt{\frac{\log p}{n}} \Bigr). 
\end{equation*}
\end{theorem}

The norm condition $\|\beta(t^*)\|_1 \le C\|\beta^*\|_1$ bridges the dynamic optimization flow with static regularization geometry and is central to reaching the minimax rate from the basic inequality. In the continuous-time $\ell_2$-boosting dynamics, the trajectory's $\ell_1$-norm grows monotonically, effectively acting as an inverse regularization parameter \citep{rosset2004boosting}. Under appropriate regularity conditions, prior work has established that stopping the boosting flow before fitting the noise floor implicitly regularizes the estimator, yielding an $\ell_1$-norm commensurate with the true signal's norm \citep{zhang2005boosting, bartlett2006adaboost}. Nevertheless, deriving this bounded $\ell_1$-norm condition from first principles under the sole assumptions of the theorem remains an open problem. The present result isolates the core mechanism, the basic inequality, and makes transparent the precise structural condition under which the early-stopped $\ell_2$-boosting flow provably attains the minimax rate.

\begin{proof}[Proof Difficulties, Sketch, and Techniques]
    The transition from late-stopping asymptotics to minimax prediction rates requires a finite-time characterization of the continuous $\ell_2$-Boosting flow. An analytical obstacle is the non-smoothness of the active selection criterion, $\rho(t) = \|g(t)\|_\infty$. Because the trajectory is governed by a differential inclusion rather than a standard gradient flow, classical calculus does not directly apply to the temporal evolution of the stopping rule.
    
    To address this non-smoothness, we analyze the subdifferential dynamics by invoking Danskin's Envelope Theorem for the $\ell_\infty$-norm. This allows us to differentiate the maximum absolute correlation almost everywhere, expressing the time derivative of $\rho(t)$ as the inner product of a subgradient $s(t)\in\partial\|g(t)\|_\infty$ and the residual velocity $\dot{g}(t)$. By the structure of the boosting flow we may take $s(t)=\dot{\beta}(t)$, leading to
    \[
    \frac{d}{dt}\rho(t) = -\frac{1}{n}\|\Phi\dot{\beta}(t)\|_2^2 \le 0.
    \]
    This proves that $\rho(t)$ is non-increasing and continuous. Together with the initial condition $\rho(0) > 2\lambda_n$ (which typically holds in the high-dimensional regime), it guarantees that the path intercepts the noise threshold $2\lambda_n$ at a unique time $t^*$, making the analytical stopping rule well-defined.
    
    With the trajectory halted at $t^*$, we bound the empirical prediction error. The stopped gradient satisfies $\|g(t^*)\|_\infty = 2\lambda_n$ by construction. A probabilistic difficulty arises from the random design matrix $\Phi$: the noise correlation $\frac{1}{n}\Phi^\top\epsilon$ is not unconditionally sub-Gaussian. By conditioning on the column-normalized $\Phi$ and marginalizing via the law of total probability, we establish the bound $\|\frac{1}{n}\Phi^\top\epsilon\|_\infty \le \lambda_n$. Applying H\"older's inequality to the inner-product expansion of the empirical risk under these limits yields:
    \[
    \frac{1}{n}\|\Phi(\beta(t^*) - \beta^*)\|_2^2 \le 3\lambda_n \|\beta(t^*) - \beta^*\|_1.
    \]
    This recovers the form of the Lasso basic inequality, derived here from the continuous dynamics rather than static KKT conditions. Under the premise that the stopped estimator satisfies $\|\beta(t^*)\|_1 \le C\|\beta^*\|_1$, the bound yields the minimax-optimal rate $\mathcal{O}(\sigma \|\beta^*\|_1\sqrt{\log p/n})$. Please refer to \textcolor{red}{Appendix \ref{app:proof_thm_early_stopping}} for the detailed proof.
\end{proof}

\begin{remark}[From Oracle to Adaptive Stopping]
\label{rem:oracle_to_adaptive}
    The stopping rule $t^*$ defined in Eq. \eqref{eq:stopping_time} depends on the oracle noise variance $\sigma^2$. In practice, $\sigma^2$ can be replaced by a consistent plug-in estimator $\hat{\sigma}^2$ using methods such as the Scaled Lasso \citep{sun2012scaled} or Refitted Cross-Validation (RCV) \citep{fan2012variance}. Let $\hat{\sigma} = \sqrt{\hat{\sigma}^2}$. Provided that the estimator is consistent such that $|\hat{\sigma} - \sigma| = o_P(1)$, the adaptive threshold $\hat{\lambda}_n = \hat{\sigma}\sqrt{(2c\log p)/n}$ preserves the conditional $\mathcal{O}(R\sigma\sqrt{\log p / n})$ prediction rate, where $R \ge \|\beta^*\|_1$. This analytical framework can substantially reduce the reliance on iterative hyperparameter tuning via grid-search, even when the noise level requires estimation.
\end{remark}

\begin{remark}[Comparison with Existing Literature]
    The connection between early-stopped boosting and $\ell_1$ regularization has been discussed extensively in terms of statistical risk bounds \citep{zhang2005boosting, bickel2009simultaneous, raskutti2011minimax}. Our analysis distinguishes itself by focusing on the continuous-time dynamic path of the subdifferential flow. By identifying the stopping condition $\rho(t^*) = 2\lambda_n$, we derive the Lasso basic inequality directly from the trajectory's properties rather than static optimization conditions. This dynamic approach explicitly bypasses the variance inflation characteristic of the interpolation limit.
\end{remark}

\section{Empirical Illustrations}
\label{sec:experiments}

In this section, we provide numerical illustrations of the theoretical phenomena established in the preceding sections. To isolate the implicit bias from step-size and finite-time artifacts, the $\ell_1$ interpolants are computed via Linear Programming (LP) using interior-point solvers. Due to the computational complexity of scaling exact LP solvers over numerous trials, we operate at moderate sample sizes (up to $n=800$). Our goal is to empirically evaluate the finite-sample scaling rates and mechanisms, rather than to extract exact infinite-dimensional constants. Additionally, we provided an exploratory experiment extending our findings to adaptive tree generation via XGBoost in \textcolor{red}{Appendix \ref{app:xgb_ver}}.

\textbf{Implementation Details.} All experiments were conducted locally using Jupyter Notebook in an Anaconda environment on Windows 10 (Version 10.0.19045). The hardware infrastructure utilized an Intel processor (Intel64 Family 6 Model 151 Stepping 2, GenuineIntel) equipped with 20 logical cores. The software stack was built on Python 3.12.4, utilizing the following key libraries: pandas 2.2.2, numpy 1.26.4, scikit-learn 1.7.2. 

\subsection{Experimental Setup}
\label{subsec:exp_setup}

In this subsection, we detailed the configurations for our numerical simulations. To systematically validate our theoretical framework, we designed a series of experiments divided into two main categories: evaluating the risk asymptotics of $\ell_1$ interpolants under isotropic and spiked-isotropic designs to verify the logarithmic variance decay (\textcolor{red}{Setups 1--3}) and analyzing the continuous-time dynamics of the $\ell_2$-Boosting flow to confirm the optimality of the analytical early stopping rule, both with oracle knowledge and adaptive variance estimation (\textcolor{red}{Setups 4--5}).

\paragraph{Setup 1. Finite-Size Scaling and Mechanism Verification under Isotropic Design (\textcolor{red}{Theorem \ref{thm:variance_lower_bound}})} 
\label{para:setup1}
We simulated pure noise interpolation ($\beta^* = \mathbf{0}, \sigma^2 = 1.0$) to examine the excess variance under an isotropic Gaussian design $\Phi \sim \mathcal{N}(0, I_p/n)$. Minimum $\ell_2$-norm interpolants were included as a comparative baseline. To address finite-sample artifacts and comprehensively evaluate the theoretical claims, the experiment is structured into three parts, with all reported metrics averaged over 20 independent trials to provide robust standard deviations (error bars):
\begin{itemize}
    \item \textit{Finite-Size Scaling:} We evaluated multiple sample sizes $n \in \{100, 200, 400, 800\}$ across the overparameterization ratio $\gamma = p/n \in [2, 30]$ to verify whether the empirical risks of both $\ell_1$ and $\ell_2$ interpolants converge to stable asymptotic trajectories.
    \item \textit{Rate Verification:} To quantitatively test the derived asymptotic rate, we performed ordinary least squares (OLS) regression mapping the empirical excess variance against the theoretical scaling component $1/\log(p/n)$. The $\ell_2$ baseline was plotted on the same axis to contrast the decay rates.
    \item \textit{Mechanism Probing:} To move beyond sparsity count constraints and directly observe noise localization, we tracked the maximum absolute feature weight $\|\hat{\beta}\|_\infty$ for both $\ell_1$ and $\ell_2$ interpolants across varying $\gamma$. All $\ell_1$ interpolants were computed using the HiGHS dual-simplex solver to ensure numerical precision.
\end{itemize}

\paragraph{Setup 2. Universality across Feature Distributions}
\label{para:setup2}
To investigate whether the logarithmic risk decay extends beyond Gaussian designs, we solved the Basis Pursuit problem for three normalized feature ensembles: 
\begin{itemize}
    \item Standard Gaussian.
    \item Rademacher ($\pm 1$ with probability $0.5$).
    \item Student-$t$ with $4$ degrees of freedom (a heavy-tailed distribution)
\end{itemize}
We scaled the sample size to $n=800$ to minimize finite-sample variations and evaluated the overparameterization ratio $\gamma = p/n \in [2, 30]$. The features were normalized to satisfy the isotropic covariance condition, and labels were generated as pure observation noise ($\sigma^2=1.0$). We subsequently performed OLS regression to test the consistency of the decay rate across these distinct distributions.

\paragraph{Setup 3. Spiked-Isotropic Designs (\textcolor{red}{Theorem \ref{thm:risk_spiked_final}})}
\label{para:setup3}
To evaluate the variance attenuation under the spiked-isotropic structure, we simulate pure noise interpolation ($\beta^* = \mathbf{0}$) with $k^* = 5$ and $\lambda_{\text{head}}=100$. Minimum $\ell_2$-norm interpolants were included as a comparative baseline. To verify the asymptotic scaling across finite system sizes, we evaluate $n \in \{100, 200, 400, 800\}$ while varying the isotropic tail ratio $\gamma = (p-k^*)/n \in [0.8, 30]$. All tail eigenvalues are set to $\lambda_{\text{tail}} = 1$ to align with the theoretical assumptions in \textcolor{red}{Theorem \ref{thm:risk_spiked_final}}. All reported metrics are averaged over 20 independent trials to provide robust standard deviations (error bars).

\paragraph{Setup 4. Continuous Path and Early Stopping (\textcolor{red}{Theorem \ref{thm:early_stopping}})}
\label{para:setup4}
To evaluate the dynamic path of the continuous $\ell_2$-Boosting flow, we simulated a high-dimensional sparse regression task with $n=800$ and $p=1600$. The true signal $\beta^*$ contained $s=5$ non-zero components of magnitude $3.0$, and the design matrix was column-normalized under isotropic Gaussian settings. We simulated the infinitesimal forward stagewise flow and tracked both the maximum absolute empirical correlation $\rho(t) = \|g(t)\|_\infty$ and the out-of-sample excess test risk across the continuous boosting iterations. To benchmark the analytical stopping rule, we included the optimum risk of $\ell_1$-penalized regression (LassoCV) and the stopping time determined by 5-fold cross-validation ($t^*_{CV}$) as empirical baselines.

\paragraph{Setup 5. Adaptive Early Stopping}
\label{para:setup5}
Building on Setup 4 ($n=800, p=1600$), we evaluated the adaptive stopping mechanism (\textcolor{red}{Remark \ref{rem:oracle_to_adaptive}}) by replacing the oracle variance $\sigma^2$ with a data-driven estimate $\hat{\sigma}^2$. The variance was estimated via the RCV estimator \citep{fan2012variance} utilizing sample splitting on the training data. This estimate defined the adaptive threshold $2\hat{\lambda}_n$ used to determine the algorithmic halt time $\hat{t}^*$. To benchmark the performance, the LassoCV optimum risk and the stopping time determined by 5-fold cross-validation ($t^*_{CV}$) were included as baselines.

\subsection{Results and Analysis}
\label{subsec:exp_results}

\begin{figure}[htbp]
    \centering
    \includegraphics[width=1.0\textwidth]{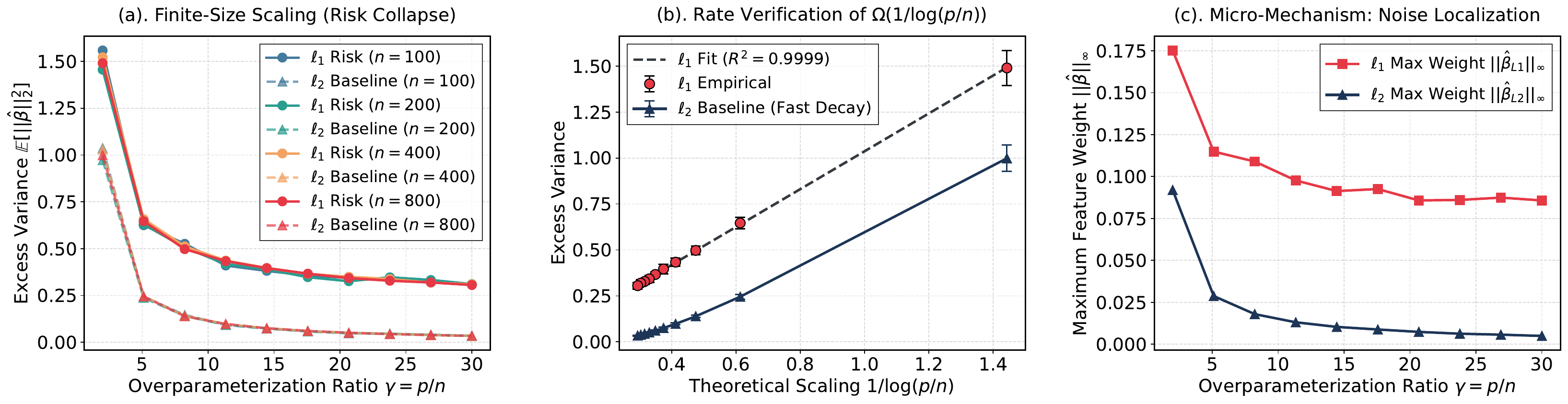}
    \caption{Empirical verification under isotropic Gaussian design. \textbf{(a)} Finite-size scaling shows risk trajectories collapsing over $\gamma = p/n$. The $\ell_2$ baseline risk decays rapidly, while the $\ell_1$ risk remains elevated. \textbf{(b)} OLS regression of the $\ell_1$ excess variance against $1/\log(p/n)$ yields $R^2 = 0.9997$. The $\ell_2$ trajectory deviates from this scaling, reflecting its faster decay rate. \textbf{(c)} Tracking the maximum feature weight reveals that $\ell_1$ interpolation localizes noise energy ($\|\hat{\beta}\|_\infty \not\to 0$), unlike the diffusion observed in $\ell_2$ interpolation.}
    \label{fig:thm1}
    \vspace{-0.8cm}
\end{figure}

\subsubsection{Logarithmic Variance Decay and Mechanism Verification}
\textcolor{red}{Figure \ref{fig:thm1}} presents the empirical evaluation for Setup 1. In \textcolor{red}{Figure \ref{fig:thm1} (a)}, the risk trajectories for different system sizes $n \in \{100, \dots, 800\}$ overlapped closely when plotted against $\gamma = p/n$, suggesting that the observed risk behavior follows high-dimensional asymptotics rather than finite-sample fluctuations. For the $\ell_2$ baseline, the excess variance decayed rapidly toward zero. In contrast, the $\ell_1$ risk decayed at a visibly slower rate, maintaining a substantially elevated profile. To quantitatively evaluate this decay rate, we regressed the empirical excess variance for $n=800$ against the scaling coordinate $1/\log(p/n)$ (\textcolor{red}{Figure \ref{fig:thm1} (b)}). The OLS fit for the $\ell_1$ interpolant yielded an $R^2$ of $0.9997$, confirming the $\Omega(1/\log(p/n))$ scaling established in \textcolor{red}{Theorem \ref{thm:variance_lower_bound}}, with tight error bars verifying high statistical stability. The $\ell_2$ baseline exhibited a distinct non-linear curve on this axis, confirming its decay is strictly faster than the $\ell_1$ logarithmic rate. Finally, to examine the noise localization mechanism, we tracked the maximum weight $\|\hat{\beta}\|_\infty$ across varying $\gamma$ (\textcolor{red}{Figure \ref{fig:thm1} (c)}). While $\|\hat{\beta}\|_\infty$ for the $\ell_2$ interpolant decayed toward zero (consistent with uniform noise diffusion), the maximum weight of the $\ell_1$ interpolant remained non-vanishing ($\approx 0.1$). This difference supports the hypothesis that greedy coordinate selection localizes noise into a sparse active set, which severely bottlenecks variance dissipation and restricts benign overfitting to a slow logarithmic rate.

\begin{figure}[htbp]
    \centering
    \includegraphics[width=0.95\textwidth]{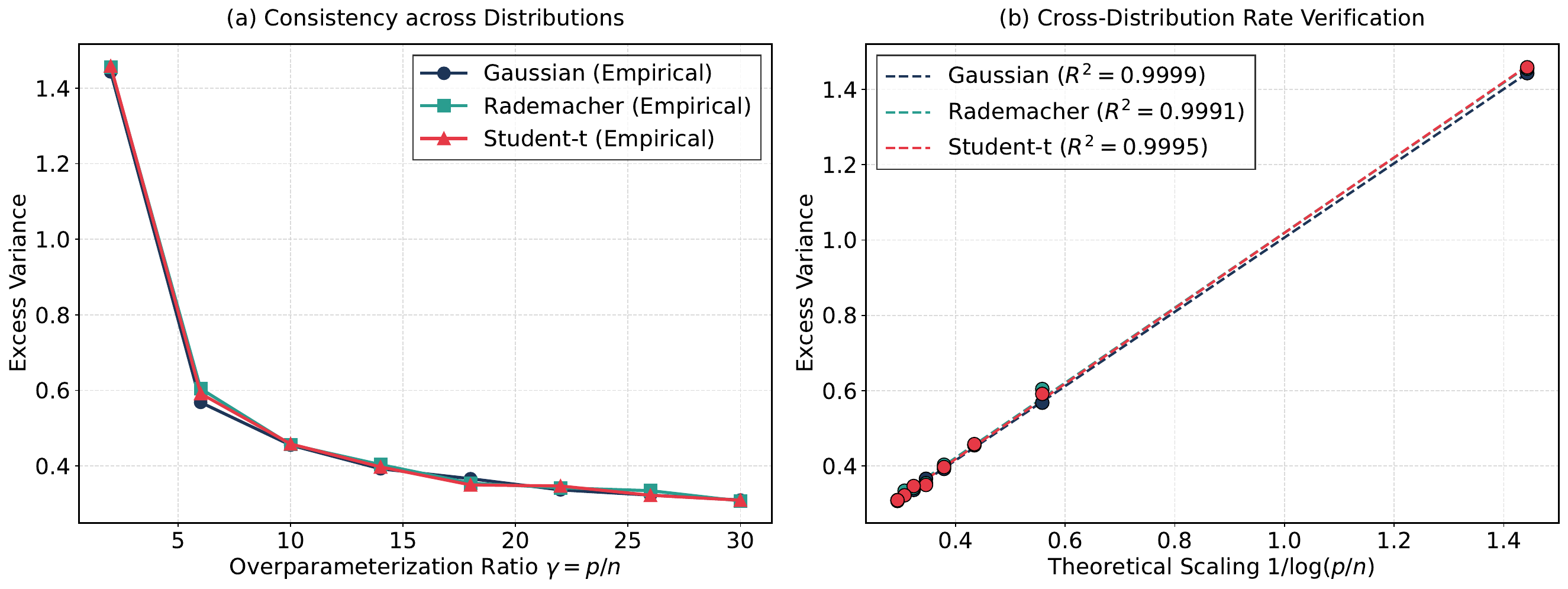}
    \caption{Empirical evaluation of $\ell_1$ risk decay across distinct feature distributions ($n=800$). \textbf{(a)} The excess variance trajectories overlap closely, showing insensitivity to feature tail properties. \textbf{(b)} OLS regressions against $1/\log(p/n)$ yield $R^2 > 0.999$ for all ensembles, verifying that the logarithmic scaling law is a universal property of the $\ell_1$ geometry.}
    \label{fig:universality}
    \vspace{-0.8cm}
\end{figure}

\begin{figure}[htbp]
    \centering
    \includegraphics[width=0.95\textwidth]{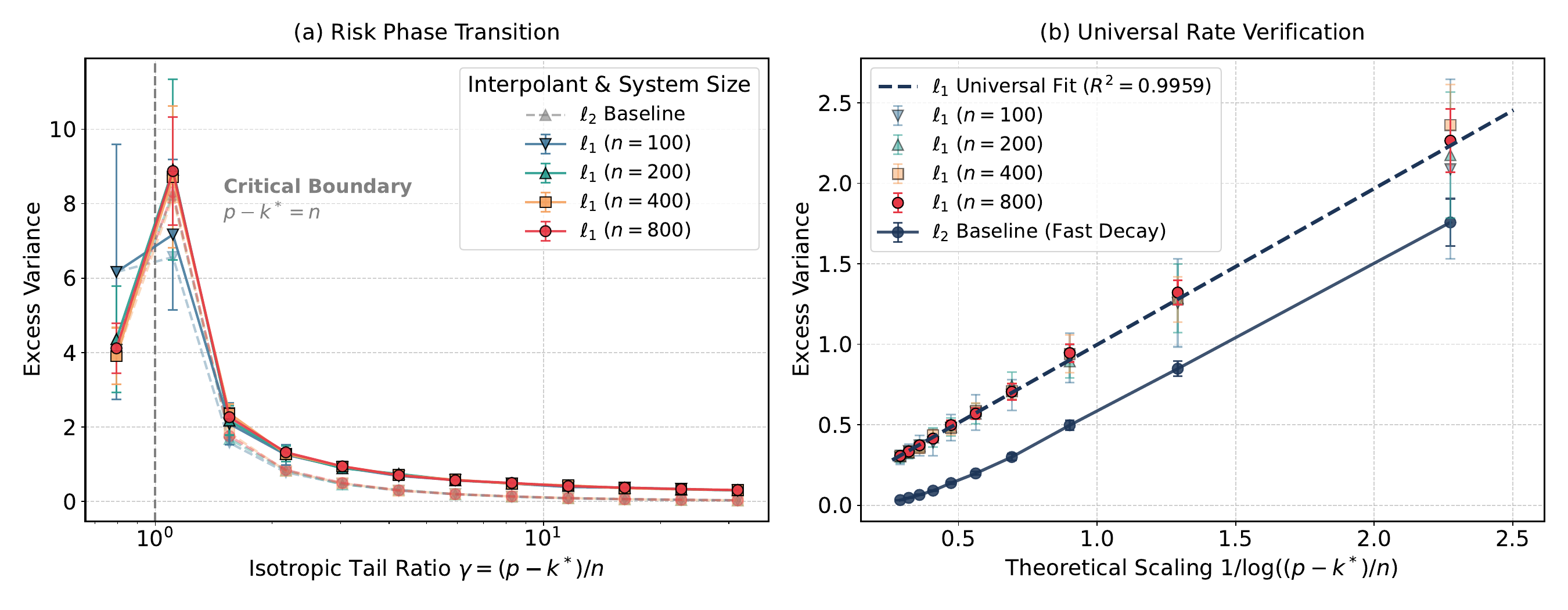} 
    \caption{Empirical excess variance under spiked-isotropic designs across multiple system sizes $n$. \textbf{(a)} Risk phase transition: variance remains elevated for $\gamma < 1$ and transitions into a decaying regime as the isotropic tail ratio $\gamma = (p-k^*)/n$ exceeds the critical boundary. The $\ell_2$ baseline exhibits rapid variance dissipation compared to $\ell_1$. \textbf{(b)} Universal rate verification: data for different $n$ collapse onto the same theoretical trajectory. A global OLS fit in the benign regime ($\gamma > 1.5$) yields $R^2 = 0.9933$, confirming the predicted $\Theta(1/\log\gamma)$ scaling for $\ell_1$ interpolants, whereas the $\ell_2$ baseline deviates with a strictly faster decay.}
    \label{fig:thm2}
    \vspace{-0.8cm}
\end{figure}

\subsubsection{Universality across Feature Distributions}
\textcolor{red}{Figure \ref{fig:universality}} illustrates the risk profiles for Setup 2. In \textcolor{red}{Figure \ref{fig:universality} 
(a)}, the excess variance trajectories for the Gaussian, Rademacher, and Student-$t$ designs overlapped closely across all evaluated overparameterization ratios. This overlap suggested that the interpolation behavior converged uniformly, unaffected by the differing tail properties of the feature distributions. To quantitatively compare the decay rates, we regressed the empirical risk against the theoretical scaling $1/\log(p/n)$ for each ensemble (\textcolor{red}{Figure \ref{fig:universality} (b)}). The regressions yielded $R^2 > 0.999$ across all three designs. This consistent linear scaling supports the conclusion that the $\mathcal{O}(1/\log(p/n))$ variance decay is a structural consequence of the $\ell_1$ optimization geometry, independent of specific sub-Gaussian or heavy-tailed feature properties.

\subsubsection{Risk Asymptotics under Spiked-Isotropic Designs}
\textcolor{red}{Figure \ref{fig:thm2}} presents the excess variance as a function of the tail ratio $\gamma$. In \textcolor{red}{Figure \ref{fig:thm2} (a)}, the risk trajectories for different system sizes $n$ overlapped closely when plotted against $\gamma$, indicating the validity of the proportional asymptotic regime. A sharp phase transition is observed at the critical boundary $\gamma = 1$ ($p-k^* = n$). For $\gamma > 1$, the isotropic tail effectively absorbs the noise interpolation capacity, enabling variance decay. In this benign regime, the $\ell_2$ baseline exhibits strictly faster variance attenuation than $\ell_1$. To verify the analytical rate, \textcolor{red}{Figure \ref{fig:thm2} (b)} plots the variance against the theoretical scaling $1/\log((p-k^*)/n)$ for the overparameterized region. The global OLS regression across all evaluated $n$ yielded an $R^2$ of $0.9959$. Supported by tight error bars, this strong linear correlation confirms that the variance attenuation is rigorously governed by the predicted logarithmic law. On this same coordinate system, the $\ell_2$ baseline deviates downward into a non-linear curve, explicitly demonstrating a distinct, faster decay profile. The empirical alignment of different $n$ onto a single $\ell_1$ fit line validates that the $\Theta(1/\log(r_2/n))$ rate is a universal property of the $\ell_1$ geometry under spiked-isotropic covariances, contrasting structurally with the $\Theta(n/r_2)$ linear decay observed in $\ell_2$ geometries.

\begin{figure}[htbp]
    \centering
    \includegraphics[width=0.85\textwidth]{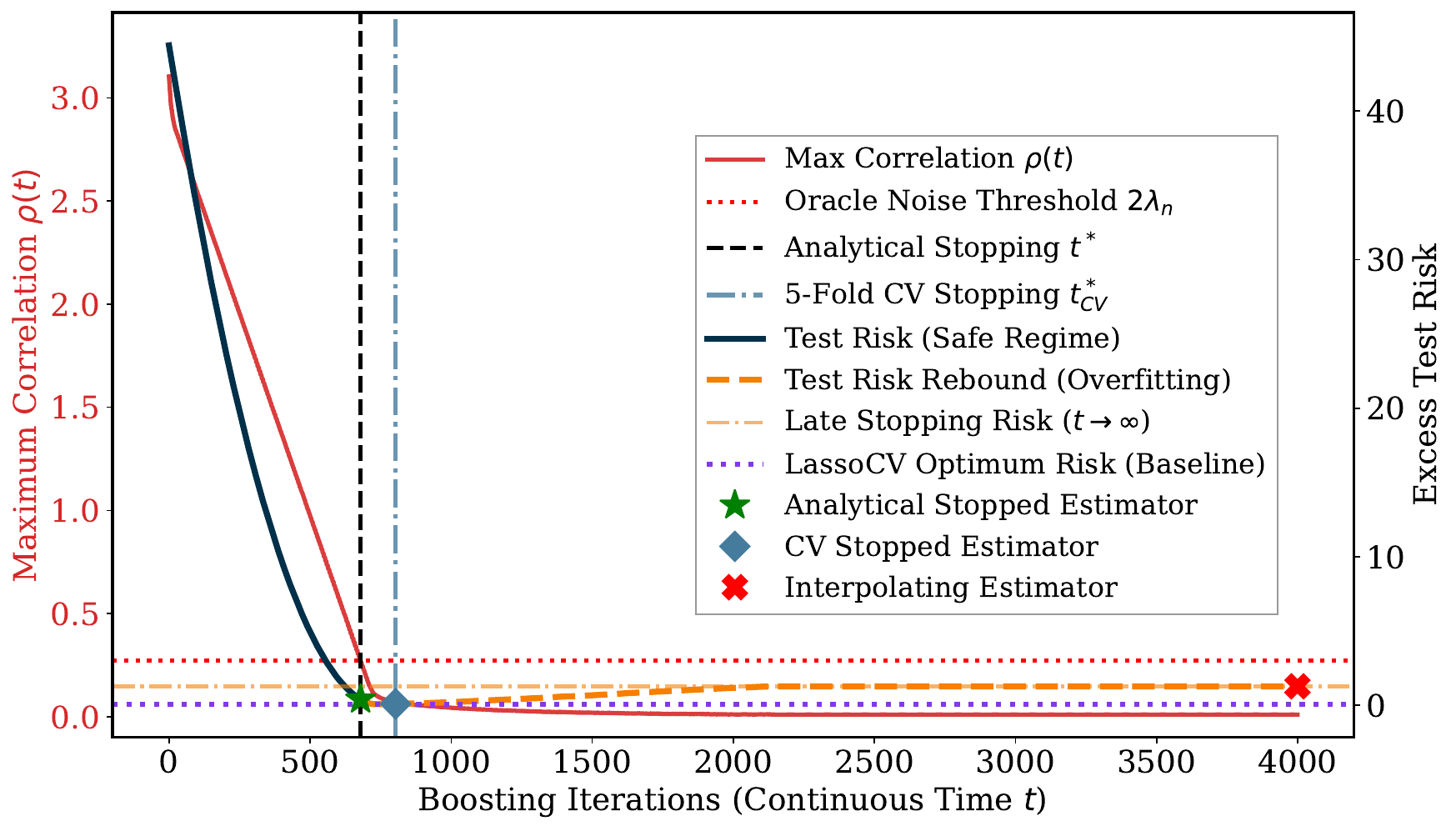}
    \caption{Dynamic path analysis of the continuous $\ell_2$-Boosting flow ($n=800, p=1600$). The analytical stopping time $t^*$ is triggered when the maximum correlation $\rho(t)$ intercepts the noise threshold $2\lambda_n$. This tuning-free stopping rule mathematically separates the safe signal-recovery regime from the subsequent noise-interpolating rebound. The analytically stopped estimator achieves the optimum risk of the LassoCV baseline and matches the performance of the 5-fold CV stopping method ($t^*_{CV}$) without the associated computational overhead.}
    \label{fig:thm3}
    \vspace{-0.8cm}
\end{figure}

\subsubsection{Minimax Rates via Early Stopping}
\textcolor{red}{Figure \ref{fig:thm3}} illustrates the temporal evolution of $\rho(t)$ and the test risk for \textcolor{red}{Setup 4}. Consistent with the subdifferential analysis, the maximum empirical correlation $\rho(t)$ decayed monotonically. The analytical stopping time $t^*$, defined by the exact intersection of $\rho(t)$ and the theoretical noise threshold $2\lambda_n$, established a deterministic termination point. Therefore, the test risk trajectory exhibited a characteristic U-shaped curve. During the initial phase, the flow recovered the active signal components, sharply reducing the risk. However, as the iterations progressed beyond the empirical minimum, the algorithm began to interpolate localized label noise, causing the test risk to inflate toward the late-stopping limit. Halting the algorithmic flow at $t^*$ successfully intercepted this overfitting rebound. The risk of the early-stopped estimator coincided with the global optimum of the LassoCV baseline, empirically supporting the implicit Lasso equivalence and the minimax prediction rates derived in \textcolor{red}{Theorem \ref{thm:early_stopping}}. Furthermore, this tuning-free analytical rule attained the identical minimal risk as the 5-fold cross-validation approach ($t^*_{CV}$) while bypassing the computational cost of iterative model refitting.

\begin{figure}[htbp]
    \centering
    \includegraphics[width=0.9\linewidth]{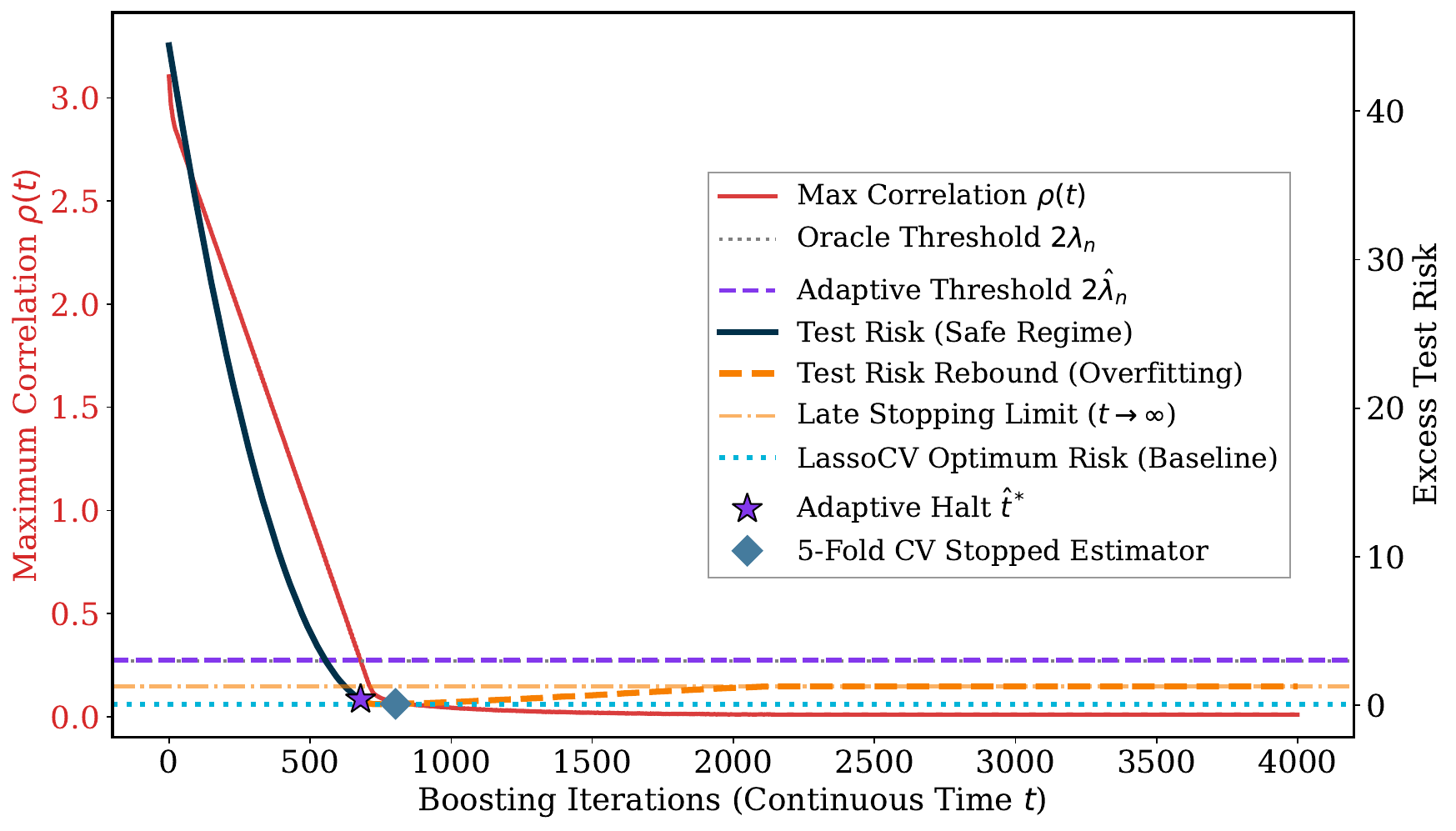}
    \caption{Adaptive early stopping of the $\ell_2$-Boosting flow ($n=800, p=1600$). The data-driven threshold $2\hat{\lambda}_n$, computed via RCV, aligns with the oracle threshold. The resulting algorithmic halt $\hat{t}^*$ intercepts the overfitting rebound. The estimator at $\hat{t}^*$ matches the optimum risk of the LassoCV baseline and exhibits predictive performance equivalent to the 5-fold CV stopped estimator without requiring repetitive model fitting.}
    \label{fig:adaptive_stopping}
    \vspace{-0.8cm}
\end{figure}

\subsubsection{Validation of the Adaptive Early Stopping Rule}
\textcolor{red}{Figure \ref{fig:adaptive_stopping}} displays the trajectory evaluation for \textcolor{red}{Setup 5}. The adaptive threshold $2\hat{\lambda}_n$ computed via RCV closely tracked the oracle threshold $2\lambda_n$, as evidenced by the overlapping horizontal lines. Consequently, the data-driven stopping time $\hat{t}^*$ intercepted the optimization flow precisely as the trajectory transitioned from the signal-recovery phase into the overfitting rebound. The test risk achieved at $\hat{t}^*$ aligned with the global minimum risk delineated by the LassoCV baseline. This confirmed that the analytical stopping criterion derived from the subdifferential flow can be fully operationalized using a consistent plug-in variance estimator. Notably, the adaptive rule attained near-optimal risk and performed equivalently to the 5-fold CV protocol, while bypassing the computational cost of iterative grid-search cross-validation.

\section{Conclusion and Discussion}
\label{sec:conclusion}

\subsection{Conclusion}
This paper characterized the generalization properties of boosting in the interpolating regime via the continuous-time $\ell_2$-Boosting flow. We established that under isotropic designs, the $\ell_1$ implicit bias prevents benign overfitting, resulting in a logarithmic variance decay of $\Theta(\sigma^2/\log(p/n))$. While spiked-isotropic heavy-tailed covariances ($ p - k^* \gg n$) enable asymptotic risk to vanish, the excess variance decays at a rate of $\Theta(\sigma^2/\log((p - k^*)/n))$. This highlights a geometric property of $\ell_1$ sparsity: greedy coordinate selection resists the uniform noise delocalization characteristic of $\ell_2$ geometries. Motivated by this slow convergence in the late-stopping regime, we derived an analytical early stopping rule based on the noise floor. Halting the flow at this threshold avoids noise localization and attains minimax prediction rates.

\subsection{Discussion: From Fixed Dictionaries to Adaptive Trees}
While our risk asymptotics apply to $\ell_2$-Boosting over a \textit{fixed dictionary} ($\Phi \in \mathbb{R}^{n \times p}$), practical boosting algorithms like XGBoost rely on \textit{adaptive feature generation}. In these algorithms, base learners (decision trees) are dynamically constructed, making the feature space data-dependent. This raises a natural question regarding the empirical alignment observed in \textcolor{red}{Section \ref{app:xgb_ver}}: \textbf{\textit{why does the adaptive generation of trees appear to inherit the logarithmic variance decay associated with the $\ell_1$ geometry?}}

\textbf{Mechanism of Noise Localization.} We posit that the adaptive nature of tree generation preserves the noise localization mechanism. In our continuous flow over a fixed dictionary, coordinate selection is governed by $\dot{\beta}(t) \in \partial \|g(t)\|_\infty$, maximizing correlation with the residual. In adaptive boosting, the tree-growing algorithm recursively searches for data splits to maximize empirical loss reduction. Conceptually, this process acts as a dual $\ell_\infty$-norm maximization over a large, dynamically generated dictionary of tree partitions. Rather than selecting from predefined features, adaptive trees construct localized indicator functions that isolate noise. Consequently, the tendency to concentrate interpolation weights into a sparse active set is intrinsic to the greedy splitting heuristic.

\textbf{Empirical Consistency of the Decay Rate.} This shared geometric bias offers an explanation for the empirical consistency between our fixed-dictionary asymptotics and the XGBoost interpolation experiment. Although the dictionary of possible trees is combinatorially large, the greedy algorithmic bias favors sparse selection over uniform energy spread. Unlike $\ell_2$ minimum-norm interpolants where noise diffuses across numerous tail dimensions, the adaptive $\ell_1$ flow interpolates noise via a sparse sequence of shallow trees. Because noise energy remains confined to this sparse support, the variance cannot dissipate at a linear rate, resulting in the slow attenuation observed empirically, which parallels the $\mathcal{O}(1/\log(\cdot))$ bound established in our analysis.

\textbf{Future Theoretical Directions.} Extending risk asymptotics from fixed random designs to adaptive, data-dependent trees is a challenging open problem in learning theory. Such an extension requires shifting from Random Matrix Theory to greedy approximation theory in infinite-dimensional functional spaces. Future work might model the boosting trajectory within variation spaces \citep{barron2002universal}, bounding the sequential Rademacher complexity of the dynamically generated splits. Nevertheless, our fixed-dictionary analysis establishes a mathematical baseline, demonstrating that the variance penalty of $\ell_1$ implicit bias is a structural consequence of greedy $\ell_\infty$ selection. The empirical evidence suggests that this mechanism persists even as the dictionary becomes adaptive.

\section*{Acknowledgements}
This research was supported by Beijing Natural Science Foundation (Z250001). The authors declare that they have no known competing financial interests or personal relationships that could have appeared to influence the work reported in this paper.

\bibliography{sample}

@article{rosset2004boosting,
  title={Boosting as a regularized path to a maximum margin classifier},
  author={Rosset, Saharon and Zhu, Ji and Hastie, Trevor},
  journal={Journal of Machine Learning Research},
  volume={5},
  number={Aug},
  pages={941--973},
  year={2004}
}

@article{efron2004least,
  title={Least angle regression},
  author={Efron, Bradley and Hastie, Trevor and Johnstone, Iain and Tibshirani, Robert},
  year={2004}
}

@inproceedings{gunasekar2018characterizing,
  title={Characterizing implicit bias in terms of optimization geometry},
  author={Gunasekar, Suriya and Lee, Jason and Soudry, Daniel and Srebro, Nathan},
  booktitle={International Conference on Machine Learning},
  pages={1832--1841},
  year={2018},
  organization={PMLR}
}

@book{rockafellar1997convex,
  title={Convex analysis},
  author={Rockafellar, R Tyrrell},
  volume={28},
  year={1997},
  publisher={Princeton university press}
}

@inproceedings{gordon2006milman,
  title={On Milman's inequality and random subspaces which escape through a mesh in {$\mathbb{R}^n$}},
  author={Gordon, Yehoram},
  booktitle={Geometric Aspects of Functional Analysis: Israel Seminar (GAFA) 1986--87},
  pages={84--106},
  year={2006},
  organization={Springer}
}

@inproceedings{thrampoulidis2015regularized,
  title={Regularized linear regression: A precise analysis of the estimation error},
  author={Thrampoulidis, Christos and Oymak, Samet and Hassibi, Babak},
  booktitle={Conference on Learning Theory},
  pages={1683--1709},
  year={2015},
  organization={PMLR}
}

@article{candes2005decoding,
  title={Decoding by linear programming},
  author={Candes, Emmanuel J and Tao, Terence},
  journal={IEEE transactions on information theory},
  volume={51},
  number={12},
  pages={4203--4215},
  year={2005},
  publisher={IEEE}
}

@article{thrampoulidis2014gaussian,
  title={The Gaussian min-max theorem in the presence of convexity},
  author={Thrampoulidis, Christos and Oymak, Samet and Hassibi, Babak},
  journal={arXiv preprint arXiv:1408.4837},
  year={2014}
}

@book{danskin2012theory,
  title={The theory of max-min and its application to weapons allocation problems},
  author={Danskin, John M},
  year={2012},
  publisher={Springer Science \& Business Media}
}

@article{stein1981estimation,
  title={Estimation of the mean of a multivariate normal distribution},
  author={Stein, Charles M},
  journal={The annals of Statistics},
  pages={1135--1151},
  year={1981},
  publisher={JSTOR}
}

@article{bartlett2020benign,
  title={Benign overfitting in linear regression},
  author={Bartlett, Peter L and Long, Philip M and Lugosi, G{\'a}bor and Tsigler, Alexander},
  journal={Proceedings of the National Academy of Sciences},
  volume={117},
  number={48},
  pages={30063--30070},
  year={2020},
  publisher={National Academy of Sciences}
}

@article{chatterji2021finite,
  title={Finite-sample analysis of interpolating linear classifiers in the overparameterized regime},
  author={Chatterji, Niladri S and Long, Philip M},
  journal={Journal of Machine Learning Research},
  volume={22},
  number={129},
  pages={1--30},
  year={2021}
}

@article{sion1958general,
  title={On general minimax theorems.},
  author={Sion, Maurice},
  year={1958}
}

@book{kolmogorov2018foundations,
  title={Foundations of the theory of probability: Second English Edition},
  author={Kolmogorov, Andrey Nikolaevich},
  year={2018},
  publisher={Courier Dover Publications}
}

@article{bickel2009simultaneous,
  title={Simultaneous analysis of Lasso and Dantzig selector},
  author={Bickel, Peter J and Ritov, Ya'acov and Tsybakov, Alexandre B},
  journal={The Annals of Statistics},
  volume={37},
  number={4},
  pages={1705--1732},
  year={2009},
  publisher={Institute of Mathematical Statistics}
}

@article{raskutti2011minimax,
  title={Minimax rates of estimation for high-dimensional linear regression over $\ell_q$-balls},
  author={Raskutti, Garvesh and Wainwright, Martin J and Yu, Bin},
  journal={IEEE transactions on information theory},
  volume={57},
  number={10},
  pages={6976--6994},
  year={2011},
  publisher={IEEE}
}

@article{zhang2005boosting,
  title={Boosting with early stopping: Convergence and consistency},
  author={Zhang, Tong and Yu, Bin},
  journal={The Annals of Statistics},
  volume={33},
  number={4},
  pages={1538--1579},
  year={2005},
  publisher={Institute of Mathematical Statistics}
}

@article{bartlett2006adaboost,
  title={Adaboost is consistent},
  author={Bartlett, Peter and Traskin, Mikhail},
  journal={Advances in Neural Information Processing Systems},
  volume={19},
  year={2006}
}

@book{ash2000probability,
  title={Probability and measure theory},
  author={Ash, Robert B and Dol{\'e}ans-Dade, Catherine A},
  year={2000},
  publisher={Academic press}
}

@article{belkin2019reconciling,
  title={Reconciling modern machine-learning practice and the classical bias--variance trade-off},
  author={Belkin, Mikhail and Hsu, Daniel and Ma, Siyuan and Mandal, Soumik},
  journal={Proceedings of the National Academy of Sciences},
  volume={116},
  number={32},
  pages={15849--15854},
  year={2019},
  publisher={National Academy of Sciences}
}

@article{hastie2022surprises,
  title={Surprises in high-dimensional ridgeless least squares interpolation},
  author={Hastie, Trevor and Montanari, Andrea and Rosset, Saharon and Tibshirani, Ryan J},
  journal={Annals of statistics},
  volume={50},
  number={2},
  pages={949},
  year={2022}
}

@article{mei2022generalization,
  title={The generalization error of random features regression: Precise asymptotics and the double descent curve},
  author={Mei, Song and Montanari, Andrea},
  journal={Communications on Pure and Applied Mathematics},
  volume={75},
  number={4},
  pages={667--766},
  year={2022},
  publisher={Wiley Online Library}
}

@article{soudry2018implicit,
  title={The implicit bias of gradient descent on separable data},
  author={Soudry, Daniel and Hoffer, Elad and Nacson, Mor Shpigel and Gunasekar, Suriya and Srebro, Nathan},
  journal={Journal of Machine Learning Research},
  volume={19},
  number={70},
  pages={1--57},
  year={2018}
}

@article{tsigler2023benign,
  title={Benign overfitting in ridge regression},
  author={Tsigler, Alexander and Bartlett, Peter L},
  journal={Journal of Machine Learning Research},
  volume={24},
  number={123},
  pages={1--76},
  year={2023}
}

@article{muthukumar2020harmless,
  title={Harmless interpolation of noisy data in regression},
  author={Muthukumar, Vidya and Vodrahalli, Kailas and Subramanian, Vignesh and Sahai, Anant},
  journal={IEEE Journal on Selected Areas in Information Theory},
  volume={1},
  number={1},
  pages={67--83},
  year={2020},
  publisher={IEEE}
}

@article{geman1992neural,
  title={Neural networks and the bias/variance dilemma},
  author={Geman, Stuart and Bienenstock, Elie and Doursat, Ren{\'e}},
  journal={Neural computation},
  volume={4},
  number={1},
  pages={1--58},
  year={1992},
  publisher={MIT Press}
}

@misc{john2010elements,
  title={The elements of statistical learning: data mining, inference, and prediction},
  author={John Lu, ZQ},
  year={2010},
  publisher={Oxford University Press}
}

@article{zhang2016understanding,
  title={Understanding deep learning requires rethinking generalization},
  author={Zhang, Chiyuan and Bengio, Samy and Hardt, Moritz and Recht, Benjamin and Vinyals, Oriol},
  journal={arXiv preprint arXiv:1611.03530},
  year={2016}
}

@article{grinsztajn2022tree,
  title={Why do tree-based models still outperform deep learning on typical tabular data?},
  author={Grinsztajn, L{\'e}o and Oyallon, Edouard and Varoquaux, Ga{\"e}l},
  journal={Advances in neural information processing systems},
  volume={35},
  pages={507--520},
  year={2022}
}

@article{friedman2001greedy,
  title={Greedy function approximation: a gradient boosting machine},
  author={Friedman, Jerome H},
  journal={Annals of statistics},
  pages={1189--1232},
  year={2001},
  publisher={JSTOR}
}

@article{shwartz2022tabular,
  title={Tabular data: Deep learning is not all you need},
  author={Shwartz-Ziv, Ravid and Armon, Amitai},
  journal={Information fusion},
  volume={81},
  pages={84--90},
  year={2022},
  publisher={Elsevier}
}

@article{buhlmann2003boosting,
  title={Boosting with the $\ell_2$ loss: regression and classification},
  author={B{\"u}hlmann, Peter and Yu, Bin},
  journal={Journal of the American Statistical Association},
  volume={98},
  number={462},
  pages={324--339},
  year={2003},
  publisher={Taylor \& Francis}
}

@article{hastie2007forward,
  title={Forward stagewise regression and the lasso},
  author={Hastie, Trevor and Taylor, Jonathan and Tibshirani, Robert and Walther, Guenther},
  journal={Electronic Journal of Statistics},
  volume={1},
  pages={1--29},
  year={2007},
  publisher={Institute of Mathematical Statistics}
}

@article{chen2001atomic,
  title={Atomic decomposition by basis pursuit},
  author={Chen, Scott Shaobing and Donoho, David L and Saunders, Michael A},
  journal={SIAM review},
  volume={43},
  number={1},
  pages={129--159},
  year={2001},
  publisher={SIAM}
}

@article{bartlett1998boosting,
  title={Boosting the margin: A new explanation for the effectiveness of voting methods},
  author={Bartlett, Peter and Freund, Yoav and Lee, Wee Sun and Schapire, Robert E},
  journal={The annals of statistics},
  volume={26},
  number={5},
  pages={1651--1686},
  year={1998},
  publisher={Institute of Mathematical Statistics}
}

@article{koehler2021uniform,
  title={Uniform convergence of interpolators: Gaussian width, norm bounds and benign overfitting},
  author={Koehler, Frederic and Zhou, Lijia and Sutherland, Danica J and Srebro, Nathan},
  journal={Advances in Neural Information Processing Systems},
  volume={34},
  pages={20657--20668},
  year={2021}
}

@article{liang2018just,
  title={Just interpolate: Kernel” ridgeless” regression can generalize. 757 arxiv e-prints p},
  author={Liang, T and Rakhlin, A},
  journal={arXiv preprint arXiv:1808.00387},
  volume={758},
  year={2018}
}

@inproceedings{telgarsky2013margins,
  title={Margins, shrinkage, and boosting},
  author={Telgarsky, Matus},
  booktitle={International Conference on Machine Learning},
  pages={307--315},
  year={2013},
  organization={PMLR}
}

@inproceedings{wang2022tight,
  title={Tight bounds for minimum $\ell_1$-norm interpolation of noisy data},
  author={Wang, Guillaume and Donhauser, Konstantin and Yang, Fanny},
  booktitle={International Conference on Artificial Intelligence and Statistics},
  pages={10572--10602},
  year={2022},
  organization={PMLR}
}

@article{miolane2018distribution,
  title={The distribution of the Lasso: Uniform control over sparse balls and adaptive parameter tuning. arXiv e-prints, page},
  author={Miolane, L{\'e}o and Montanari, Andrea},
  journal={arXiv preprint arXiv:1811.01212},
  year={2018}
}

@inproceedings{taheri2021fundamental,
  title={Fundamental limits of ridge-regularized empirical risk minimization in high dimensions},
  author={Taheri, Hossein and Pedarsani, Ramtin and Thrampoulidis, Christos},
  booktitle={International Conference on Artificial Intelligence and Statistics},
  pages={2773--2781},
  year={2021},
  organization={PMLR}
}

@article{celentano2023lasso,
  title={The lasso with general gaussian designs with applications to hypothesis testing},
  author={Celentano, Michael and Montanari, Andrea and Wei, Yuting},
  journal={The Annals of Statistics},
  volume={51},
  number={5},
  pages={2194--2220},
  year={2023},
  publisher={Institute of Mathematical Statistics}
}

@article{sun2012scaled,
  title={Scaled sparse linear regression},
  author={Sun, Tingni and Zhang, Cun-Hui},
  journal={Biometrika},
  volume={99},
  number={4},
  pages={879--898},
  year={2012},
  publisher={Oxford University Press}
}

@article{fan2012variance,
  title={Variance estimation using refitted cross-validation in ultrahigh dimensional regression},
  author={Fan, Jianqing and Guo, Shaojun and Hao, Ning},
  journal={Journal of the Royal Statistical Society Series B: Statistical Methodology},
  volume={74},
  number={1},
  pages={37--65},
  year={2012},
  publisher={Oxford University Press}
}

@article{barron2002universal,
  title={Universal approximation bounds for superpositions of a sigmoidal function},
  author={Barron, Andrew R},
  journal={IEEE Transactions on Information theory},
  volume={39},
  number={3},
  pages={930--945},
  year={2002},
  publisher={IEEE}
}

@article{cao2022benign,
  title={Benign overfitting in two-layer convolutional neural networks},
  author={Cao, Yuan and Chen, Zixiang and Belkin, Misha and Gu, Quanquan},
  journal={Advances in neural information processing systems},
  volume={35},
  pages={25237--25250},
  year={2022}
}

@inproceedings{kou2023benign,
  title={Benign overfitting in two-layer relu convolutional neural networks},
  author={Kou, Yiwen and Chen, Zixiang and Chen, Yuanzhou and Gu, Quanquan},
  booktitle={International conference on machine learning},
  pages={17615--17659},
  year={2023},
  organization={PMLR}
}

@inproceedings{zhu2023benign,
  title={Benign overfitting in deep neural networks under lazy training},
  author={Zhu, Zhenyu and Liu, Fanghui and Chrysos, Grigorios and Locatello, Francesco and Cevher, Volkan},
  booktitle={International Conference on Machine Learning},
  pages={43105--43128},
  year={2023},
  organization={PMLR}
}

@article{kim2025transfer,
  title={Transfer Learning for Benign Overfitting in High-Dimensional Linear Regression},
  author={Kim, Yeichan and Kim, Ilmun and Park, Seyoung},
  journal={arXiv preprint arXiv:2510.15337},
  year={2025}
}

@article{frei2024trained,
  title={Trained transformer classifiers generalize and exhibit benign overfitting in-context},
  author={Frei, Spencer and Vardi, Gal},
  journal={arXiv preprint arXiv:2410.01774},
  year={2024}
}

@article{jiang2024unveil,
  title={Unveil benign overfitting for transformer in vision: Training dynamics, convergence, and generalization},
  author={Jiang, Jiarui and Huang, Wei and Zhang, Miao and Suzuki, Taiji and Nie, Liqiang},
  journal={Advances in Neural Information Processing Systems},
  volume={37},
  pages={135464--135625},
  year={2024}
}

@article{magen2024benign,
  title={Benign overfitting in single-head attention},
  author={Magen, Roey and Shang, Shuning and Xu, Zhiwei and Frei, Spencer and Hu, Wei and Vardi, Gal},
  journal={arXiv preprint arXiv:2410.07746},
  year={2024}
}

@article{sakamoto2024benign,
  title={Benign Overfitting in Token Selection of Attention Mechanism},
  author={Sakamoto, Keitaro and Sato, Issei},
  journal={arXiv preprint arXiv:2409.17625},
  year={2024}
}

@article{shang2024initialization,
  title={Initialization matters: On the benign overfitting of two-layer relu cnn with fully trainable layers},
  author={Shang, Shuning and Meng, Xuran and Cao, Yuan and Zou, Difan},
  journal={arXiv preprint arXiv:2410.19139},
  year={2024}
}

@article{xu2025rethinking,
  title={Rethinking benign overfitting in two-layer neural networks},
  author={Xu, Ruichen and Chen, Kexin},
  journal={arXiv preprint arXiv:2502.11893},
  year={2025}
}

@article{mallinar2207benign,
  title={Benign, tempered, or catastrophic: A taxonomy of overfitting},
  author={Mallinar, Neil and Simon, James B and Abedsoltan, Amirhesam and Pandit, Parthe and Belkin, Mikhail and Nakkiran, Preetum},
  journal={URL https://arxiv. org/abs/2207},
  volume={6569},
  year={2022}
}

@book{wainwright2019high,
  title={High-dimensional statistics: A non-asymptotic viewpoint},
  author={Wainwright, Martin J},
  volume={48},
  year={2019},
  publisher={Cambridge university press}
}

@article{tibshirani2013lasso,
  title={The lasso problem and uniqueness},
  author={Tibshirani, Ryan J},
  journal={Electronic Journal of Statistics},
  volume={7},
  pages={1456--1490},
  year={2013}
}

@article{scornet2016random,
  title={Random forests and kernel methods},
  author={Scornet, Erwan},
  journal={IEEE Transactions on Information Theory},
  volume={62},
  number={3},
  pages={1485--1500},
  year={2016},
  publisher={IEEE}
}

@article{biau2016random,
  title={A random forest guided tour},
  author={Biau, G{\'e}rard and Scornet, Erwan},
  journal={Test},
  volume={25},
  number={2},
  pages={197--227},
  year={2016},
  publisher={Springer}
}

@inproceedings{rahaman2019spectral,
  title={On the spectral bias of neural networks},
  author={Rahaman, Nasim and Baratin, Aristide and Arpit, Devansh and Draxler, Felix and Lin, Min and Hamprecht, Fred and Bengio, Yoshua and Courville, Aaron},
  booktitle={International conference on machine learning},
  pages={5301--5310},
  year={2019},
  organization={PMLR}
}

@article{dobriban2018high,
  title={High-dimensional asymptotics of prediction: Ridge regression and classification},
  author={Dobriban, Edgar and Wager, Stefan},
  journal={The Annals of Statistics},
  volume={46},
  number={1},
  pages={247--279},
  year={2018},
  publisher={JSTOR}
}

@article{donoho1995adapting,
  title={Adapting to unknown smoothness via wavelet shrinkage},
  author={Donoho, David L and Johnstone, Iain M},
  journal={Journal of the american statistical association},
  volume={90},
  number={432},
  pages={1200--1224},
  year={1995},
  publisher={Taylor \& Francis}
}

@book{feller1968introduction,
  title={An introduction to probability theory and its applications, Volume 1},
  author={Feller, William},
  year={1968},
  publisher={John Wiley \& Sons}
}

@misc{abramowitz1966handbook,
  title={Handbook of mathematical functions, with formulas, graphs, and mathematical tables},
  author={Abramowitz, Milton and Stegun, Irene A and Romain, Jacques E},
  year={1966},
  publisher={American Institute of Physics}
}

@article{bai1993limit,
  title={Limit of the smallest eigenvalue of a large dimensional sample covariance matrix},
  author={Bai, Zhi-Dong and Yin, Yong-Qua and others},
  journal={Ann. Probab},
  volume={21},
  number={3},
  pages={1275--1294},
  year={1993},
  publisher={World Scientific}
}

@book{muirhead2009aspects,
  title={Aspects of multivariate statistical theory},
  author={Muirhead, Robb J},
  year={2009},
  publisher={John Wiley \& Sons}
}

@misc{vershynin2012introduction,
  title={Introduction to the non-asymptotic analysis of random matrices.},
  author={Vershynin, Roman},
  year={2012}
}

@article{laurent2000adaptive,
  title={Adaptive estimation of a quadratic functional by model selection},
  author={Laurent, Beatrice and Massart, Pascal},
  journal={Annals of statistics},
  pages={1302--1338},
  year={2000},
  publisher={JSTOR}
}

@article{von1988moments,
  title={Moments for the inverted Wishart distribution},
  author={Von Rosen, Dietrich},
  journal={Scandinavian Journal of Statistics},
  pages={97--109},
  year={1988},
  publisher={JSTOR}
}

@article{su2026itboost,
  title={ITBoost: Information-Theoretic Trust for Robust Boosting},
  author={Su, Ye and Zhao, Longlong and Garcia-Gil, Diego and Guo, Jipeng and Zhang, Gangchun and Chen, Jinxin and Chen, Jinsong},
  journal={arXiv preprint arXiv:2605.04671},
  year={2026}
}


\newpage

\appendix

\begin{longtable}{p{0.15\textwidth} p{0.65\textwidth} p{0.15\textwidth}}
\caption{Summary of Symbols and Notations}
\label{tab:global_notation} \\
\toprule
\textbf{Symbol} & \textbf{Definition} & \textbf{Domain/Space} \\
\midrule
\endfirsthead

\multicolumn{3}{c}%
{{\bfseries \tablename\ \thetable{} -- continued from previous page}} \\
\toprule
\textbf{Symbol} & \textbf{Definition} & \textbf{Domain/Space} \\
\midrule
\endhead

\midrule \multicolumn{3}{r}{{Continued on next page}} \\ \midrule
\endfoot

\bottomrule
\endlastfoot

\multicolumn{3}{l}{\textit{\textbf{1. Global Dimensions and Asymptotics}}} \\
\midrule
$n, p$ & Number of training samples and number of features & $\mathbb{Z}^+$ \\
$m$ & Number of base learners in the dictionary & $\mathbb{N}^+$ \\
$\gamma$ & Overparameterization ratio, defined as the limit of $p/n$ & $(1, \infty]$ \\
$s$ & Sparsity level of the true target function ($s = \|\beta^*\|_0$) & $\mathbb{Z}^+$ \\
$k^*$ & Critical split index partitioning the head and tail of the covariance spectrum & $\mathbb{Z}^+$ \\
$[n], [p]$ & Index sets $\{1, \dots, n\}$ and $\{1, \dots, p\}$ & $\mathbb{Z}^+$ \\
$\mathcal{O}, o, \Omega, \Theta$ & Standard big-$\mathcal{O}$, little-$o$, big-$\Omega$, and big-$\Theta$ deterministic asymptotic notations & - \\
$\mathcal{O}_P, o_P$ & Stochastic boundedness and convergence in probability notations & - \\
$\xrightarrow{\mathbb{P}}$ & Convergence in probability as $n, p \to \infty$ & - \\

\midrule
\multicolumn{3}{l}{\textit{\textbf{2. Data Generation and Feature Space}}} \\
\midrule
$\mathcal{X}$ & Input feature space & - \\
$\mathcal{D}$ & Joint probability distribution of the data & - \\
$S$ & Training dataset $S = \{(x_i, y_i)\}_{i=1}^n$ drawn i.i.d. from $\mathcal{D}$ & $\mathcal{X} \times \mathbb{R}$ \\
$\mathcal{H}$ & Dictionary of base hypotheses (weak learners), $\mathcal{H} = \{h_1, \dots, h_p\}$ & - \\
$\phi(x)$ & Feature map vector evaluating input $x$ on all base learners & $\mathbb{R}^p$ \\
$\Phi$ & Design matrix, where the $i$-th row is $\phi(x_i)^\top \sim \mathcal{N}(0, \Sigma)$ & $\mathbb{R}^{n \times p}$ \\
$Y$ & Response vector composed of labels $y_i$ & $\mathbb{R}^n$ \\
$\beta^*$ & Unknown true ensemble weight vector & $\mathbb{R}^p$ \\
$\mathcal{B}_0(s)$ & The class of exactly $s$-sparse signals, $\{\beta \in \mathbb{R}^p : \|\beta\|_0 \le s\}$ & - \\
$\epsilon_i, \epsilon$ & Observation label noise scalar and noise vector, $\epsilon \sim \mathcal{N}(0, \sigma^2 I_n)$ & $\mathbb{R}, \mathbb{R}^n$ \\
$\sigma^2$ & Variance of the label noise ($\mathbb{E}[\epsilon_i^2] = \sigma^2$) & $\mathbb{R}_{>0}$ \\

\midrule
\multicolumn{3}{l}{\textit{\textbf{3. Covariance Geometry and Spectrum}}} \\
\midrule
$\Sigma$ & Base-learner feature covariance matrix, $\Sigma = \mathbb{E}_x[\phi(x)\phi(x)^\top]$ & $\mathbb{R}^{p \times p}$ \\
$\Sigma^{1/2}$ & Unique symmetric positive semi-definite square root of $\Sigma$ & $\mathbb{R}^{p \times p}$ \\
$\lambda_i$ & The $i$-th largest eigenvalue of the covariance matrix $\Sigma$ & $\mathbb{R}_{\ge 0}$ \\
$I_p, I_n$ & Identity matrices of dimension $p$ and $n$ & $\mathbb{R}^{p \times p}, \mathbb{R}^{n \times n}$ \\
$H, T$ & Sets of indices for the head (spikes) and the tail components & - \\
$\Sigma_T$ & Tail covariance matrix component & $\mathbb{R}^{(p-k^*) \times (p-k^*)}$ \\
$\lambda_{\text{tail}}$ & Constant eigenvalue associated with the isotropic tail covariance & $\mathbb{R}_{\ge 0}$ \\
$r_1(\Sigma_T)$ & Generalized $\ell_1$ effective rank (trace-to-spectral-norm ratio of the tail) & $\mathbb{R}_{\ge 0}$ \\
$r_2(\Sigma_T)$ & Generalized $\ell_2$ effective rank of the tail (effective isotropic dimensions) & $\mathbb{R}_{\ge 0}$ \\

\midrule
\multicolumn{3}{l}{\textit{\textbf{4. Continuous-Time Boosting Dynamics}}} \\
\midrule
$t$ & Continuous time parameter of the infinitesimal flow & $\mathbb{R}_{\ge 0}$ \\
$\beta(t)$ & Continuous-time trajectory of the ensemble weight vector & $\mathbb{R}^p$ \\
$\mathcal{B}_1$ & Unit $\ell_1$-ball in $\mathbb{R}^p$, defined as $\{u \in \mathbb{R}^p : \|u\|_1 \le 1\}$ & Set \\
$u$ & Dummy variable for optimization or search over the feature space & $\mathbb{R}^p$ \\
$\hat{L}(\beta)$ & Empirical squared risk, $\frac{1}{2n}\|Y - \Phi\beta\|_2^2$ & $\mathbb{R}_{\ge 0}$ \\
$r(t)$ & Residual vector at time $t$, $r(t) = Y - \Phi\beta(t)$ & $\mathbb{R}^n$ \\
$g(t)$ & Negative gradient of the empirical risk (residual correlation vector) & $\mathbb{R}^p$ \\
$\rho(t)$ & Maximum absolute empirical correlation, $\rho(t) = \|g(t)\|_\infty$ & $\mathbb{R}_{\ge 0}$ \\
$e_i$ & The $i$-th standard basis vector & $\mathbb{R}^p$ \\
$\hat{\beta}_\infty$ & Basis Pursuit Interpolant (stationary point of the late-stopping flow) & $\mathbb{R}^p$ \\
$t^*$ & Analytical early stopping time driven by the pure noise threshold & $\mathbb{R}_{\ge 0}$ \\
$\lambda_n$ & Analytical threshold of pure noise correlation, $\lambda_n = \sigma \sqrt{2c\log p / n}$ & $\mathbb{R}_{>0}$ \\
$\hat{\sigma}^2, \hat{\lambda}_n, \hat{t}^*$ & Data-driven noise variance estimate, adaptive threshold, and adaptive stopping time & $\mathbb{R}_{>0}, \mathbb{R}_{>0}, \mathbb{R}_{\ge 0}$ \\

\midrule
\multicolumn{3}{l}{\textit{\textbf{5. Risks and Theoretical Bounds}}} \\
\midrule
$\mathcal{E}(\hat{\beta})$ & Expected out-of-sample excess risk, $\mathbb{E}_{x,y}[(y - \langle \hat{\beta}, \phi(x) \rangle)^2] - \sigma^2$ & $\mathbb{R}_{\ge 0}$ \\
$\mathcal{E}_{\text{head}}, \mathcal{E}_{\text{tail}}$ & Decomposed excess risk components corresponding to the head and tail spectra & $\mathbb{R}_{\ge 0}$ \\
$R$ & Radius bound of the $\ell_1$-ball containing the true signal, $\|\beta^*\|_1 \le R$ & $\mathbb{R}_{\ge 0}$ \\
$C_1, C_2, c$ & Universal absolute constants & $\mathbb{R}_{>0}$ \\

\midrule
\multicolumn{3}{l}{\textit{\textbf{6. CGMT and Appendix Auxiliary Variables}}} \\
\midrule
$\text{PO}, \text{AO}$ & Primary Optimization and Auxiliary Optimization in the CGMT framework & - \\
$Z$ & Standard Gaussian matrix with i.i.d. $\mathcal{N}(0,1)$ entries & $\mathbb{R}^{n \times p}$ \\
$g, h$ & Independent standard Gaussian vectors in CGMT Auxiliary Optimization & $\mathbb{R}^p, \mathbb{R}^n$ \\
$v, \lambda$ & Primal and dual optimization variables in the min-max formulation & $\mathbb{R}^p, \mathbb{R}^n$ \\
$\mathcal{C}_R, \mathcal{C}_{R_n}$ & Compact constraint sets constructed to guarantee uniform integrability in CGMT & Set \\
$\hat{v}_R, \hat{w}_{R_n}$ & Constrained estimators optimized over the compact sets $\mathcal{C}_R$ and $\mathcal{C}_{R_n}$ & $\mathbb{R}^p$ \\
$v^*, v_{\text{feas}}$ & Optimal primal solution in AO and the constructed feasible pseudo-inverse solution & $\mathbb{R}^p$ \\
$\hat{w}_n, w$ & Estimation error vectors defined as $\hat{w}_n = \hat{\beta}_\infty - \beta^*$ and $w = v - \beta^*$ & $\mathbb{R}^p$ \\
$x_i, w_i$ & Scalar components utilized in the coordinate-wise separable AO decoupling & $\mathbb{R}$ \\
$Z^\dagger$ & Moore-Penrose pseudo-inverse of the Gaussian matrix $Z$ & $\mathbb{R}^{p \times n}$ \\
$\psi(v, \mu)$ & Convex-concave objective function formulated in the primary CGMT problem & $\mathbb{R}$ \\
$f_n(v)$ & Stochastic objective function sequence in the Auxiliary Optimization & $\mathbb{R}$ \\
$\tau^*$ & Optimal dual scalar solution in CGMT calibration determining threshold intensity & $\mathbb{R}_{>0}$ \\
$\tau, \tilde{\tau}$ & Rescaled scalar dual variables to track bounded subsets & $\mathbb{R}_{\ge 0}$ \\
$\mu$ & Unit vector representing the direction of the dual variable $\lambda$ ($\mu = \lambda / \|\lambda\|_2$) & $\mathbb{R}^n$ \\
$\beta$ (scalar) & Variational scalar variable utilized in the CGMT decoupling (Appendix B) & $\mathbb{R}_{>0}$ \\
$\alpha$ & Deterministic limit of the unconstrained estimator's norm ($\|\hat{v}_R\|_2 \xrightarrow{\mathbb{P}} \alpha$) & $\mathbb{R}_{\ge 0}$ \\
$\zeta$ & Local standard deviation proxy variable, $\zeta = \frac{1}{\sqrt{n}}\|\Sigma^{1/2}\nu\|_2$ & $\mathbb{R}_{\ge 0}$ \\
$b, b^*$ & Variational optimization variable from the AM-GM identity, and its optimal deterministic scalar solution (e.g., $b^* = \sqrt{\sigma^2 + \alpha^2}$) & $\mathbb{R}_{>0}$ \\
$c_i$ & Coupled signal magnitude parameter over the head dimensions & $\mathbb{R}$ \\
$\nu$ & Shifted estimation error vector tracking the true signal deviation, $\nu = v - \beta^*$ & $\mathbb{R}^p$ \\
$\theta_n$ & Deterministic scalar sequence uniquely solving the degrees-of-freedom calibration & $\mathbb{R}_{\ge 0}$ \\
$\kappa, \kappa_T, \kappa_H$ & Degrees-of-freedom calibration thresholds in AO over different spectral regions & $\mathbb{R}_{>0}$ \\
$P_T, V$ & Tail probability mass $\mathbb{P}(|g| > \kappa_T)$ and total normalized variance & $\mathbb{R}_{\ge 0}$ \\
$V_H, V_T$ & Decomposed variance components corresponding to the head and tail spectra & $\mathbb{R}_{\ge 0}$ \\
$s(t)$ & Subgradient vector governing the flow dynamics, $s(t) \in \partial\|g(t)\|_\infty$ & $\mathbb{R}^p$ \\
$w, \mathcal{E}$ & Noise correlation vector $w = \frac{1}{n}\Phi^\top \epsilon$ and the associated high-probability event & $\mathbb{R}^p, \text{Event}$ \\

\midrule
\multicolumn{3}{l}{\textit{\textbf{7. Mathematical Operators and Standard Functions}}} \\
\midrule
$\partial f(x)$ & Subdifferential set of a convex function $f$ at point $x$ & Set \\
$\text{conv}(\cdot)$ & Convex hull operator characterizing the subdifferential of the $\ell_\infty$-norm & Set \\
$\text{sgn}(\cdot)$ & Sign function extracting the sign of a real number & $\{-1, 0, 1\}$ \\
$(\cdot)_+$ & ReLU (positive part) operator, $(x)_+ = \max(x, 0)$ & $\mathbb{R}_{\ge 0}$ \\
$\sigma_{\min}(\cdot)$ & Minimum singular value of a matrix & $\mathbb{R}_{\ge 0}$ \\
$\eta(x; \kappa)$ & Soft-thresholding function (proximal operator for $\ell_1$), $\eta(x; \kappa) = \text{sgn}(x)(|x|-\kappa)_+$ & $\mathbb{R}$ \\
$Q(\cdot)$ & Standard Gaussian tail probability function (Q-function), $Q(x) = \mathbb{P}(Z > x)$ & $[0, 1]$ \\
$\varphi(\cdot)$ & Probability density function (PDF) of the standard normal distribution & $\mathbb{R}_{>0}$ \\
$\operatorname{erfc}(\cdot)$ & Complementary error function, $\operatorname{erfc}(x) = \frac{2}{\sqrt{\pi}}\int_x^\infty e^{-t^2} dt$ & $[0, 2]$ \\
$\langle \cdot, \cdot \rangle$ & Standard Euclidean inner product & $\mathbb{R}$ \\
\end{longtable}

\section{Missing Proof of Theorem~\ref{thm:variance_lower_bound}}
\label{app:proof_thm_variance}

\begin{proof}
Let $Z \in \mathbb{R}^{n \times p}$ be the design matrix with i.i.d.\ $\mathcal{N}(0,1)$ entries and let $\epsilon \sim \mathcal{N}(0,\sigma^2 I_n)$ be the noise vector. Under the pure noise model $\beta^* = 0$, the Basis Pursuit interpolant is
\begin{equation}
\label{eq:bp_def}
\hat{v}_n \in \arg\min_{v \in \mathbb{R}^p} \|v\|_1 \quad \text{s.t.} \quad Z v = \epsilon.
\end{equation}
The asymptotic regime is $n, p \to \infty$ with $p/n \to \infty$. Because the design is isotropic, the excess risk of $\hat{\beta}_\infty = \hat{v}_n$ equals $\mathbb{E}\|\hat{v}_n\|_2^2$.

Let $Z^\dagger$ be the Moore-Penrose pseudo-inverse of $Z$. The vector $\tilde{v} = Z^\dagger \epsilon$ is feasible for Eq.~\eqref{eq:bp_def}. By the Bai-Yin law \citep{bai1993limit}, for $p/n \to \infty$, $\sigma_{\min}(Z) \sim \sqrt{p}$ almost surely. Thus, the operator norm satisfies $\|Z^\dagger\|_2 = 1/\sigma_{\min}(Z) = \mathcal{O}_{\mathbb{P}}(p^{-1/2})$. Moreover, $\|\epsilon\|_2 = \mathcal{O}_{\mathbb{P}}(\sqrt{n})$. Therefore,
\begin{equation*}
\|\tilde{v}\|_1 \le \sqrt{p}\,\|\tilde{v}\|_2 \le \sqrt{p}\,\|Z^\dagger\|_2\,\|\epsilon\|_2 = \mathcal{O}_{\mathbb{P}}(\sqrt{n}).
\end{equation*}
Because $\hat{v}_n$ minimizes the $\ell_1$-norm over the feasible set, it satisfies:
\begin{equation*}
\label{eq:apriori_bound}
\|\hat{v}_n\|_2 \le \|\hat{v}_n\|_1 \le \|\tilde{v}\|_1 = \mathcal{O}_{\mathbb{P}}(\sqrt{n}).
\end{equation*}
Choose a deterministic sequence $R_n = C_0 \sqrt{n}$ with $C_0$ sufficiently large such that $\mathbb{P}\bigl( \|\hat{v}_n\|_2 \le R_n \bigr) \to 1$. Define the compact set $\mathcal{C}_{R_n} = \{ v \in \mathbb{R}^p : \|v\|_2 \le R_n\}$ and let $\hat{v}_{R_n}$ be the solution of Eq.~\eqref{eq:bp_def} restricted to $\mathcal{C}_{R_n}$. On the high-probability event $\{\|\hat{v}_n\|_2 \le R_n\}$, the constraint is inactive, ensuring $\hat{v}_{R_n} = \hat{v}_n$.

Introducing a dual variable $\mu \in \mathbb{R}^n$, the constrained problem is written as the Primary Optimization (PO):
\begin{equation}
\label{eq:PO}
\min_{v \in \mathcal{C}_{R_n}} \max_{\mu \in \mathbb{R}^n} \bigl\{ \mu^\top(\epsilon - Z v) + \|v\|_1 \bigr\}.
\end{equation}
The objective equals $-\mu^\top Z v + \psi(v,\mu)$ with $\psi(v,\mu) = \|v\|_1 + \mu^\top\epsilon$, which is convex in $v$, concave in $\mu$, continuous, and independent of $Z$. By the CGMT \citep{gordon2006milman,thrampoulidis2015regularized}, the asymptotic behavior of Eq.~\eqref{eq:PO} is equivalent to that of the Auxiliary Optimization (AO):
\begin{equation}
\label{eq:AO}
\min_{v \in \mathcal{C}_{R_n}} \max_{\tau \ge 0}
\Bigl\{ \|v\|_1 + \tau \bigl\| \epsilon - \|v\|_2 h \bigr\|_2 - \tau g^\top v \Bigr\},
\end{equation}
where $g \sim \mathcal{N}(0,I_p)$ and $h \sim \mathcal{N}(0,I_n)$ are independent standard Gaussian vectors.

To exchange the limits and apply uniform convergence, we first establish that the optimal dual variable $\tau^*$ is stochastically bounded ($\mathcal{O}_{\mathbb{P}}(1)$). Since the objective is linear in $\tau$ and convex in $v$ over the compact set $\mathcal{C}_{R_n}$, Sion's Minimax Theorem \citep{sion1958general} allows us to exchange the operators to $\max_{\tau \ge 0} \min_{v \in \mathcal{C}_{R_n}}$. Define the inner minimum as $\psi_n(\tau)$. We upper bound $\psi_n(\tau)$ by evaluating the objective at a specific test point $v_{\text{test}} = R_n \frac{g}{\|g\|_2} \in \mathcal{C}_{R_n}$:
\begin{equation*}
\psi_n(\tau) \le \|v_{\text{test}}\|_1 + \tau \bigl\| \epsilon - \|v_{\text{test}}\|_2 h \bigr\|_2 - \tau g^\top v_{\text{test}}.
\end{equation*}
By standard Gaussian concentration, $\|v_{\text{test}}\|_1 \sim R_n \sqrt{2p/\pi}$ and $g^\top v_{\text{test}} = R_n \|g\|_2 \sim R_n \sqrt{p}$ with high probability. Meanwhile, the residual norm scales as 
\begin{equation*}
\| \epsilon - R_n h \|_2 = \mathcal{O}_{\mathbb{P}}(\sqrt{n} R_n).
\end{equation*}
Thus, the upper bound is asymptotically dominated by 
\begin{equation*}
R_n \sqrt{p} (\sqrt{2/\pi} - \tau) + \tau \mathcal{O}_{\mathbb{P}}(\sqrt{n} R_n).
\end{equation*}
Since $p/n \to \infty$, the $\sqrt{p}$ term dominates the $\sqrt{n}$ term. For any fixed $\tau > \sqrt{2/\pi}$, $\psi_n(\tau) \to -\infty$. Because the optimal value of the PO (and thus the AO) is bounded, the maximizer must satisfy $\tau^* \le T_{\max}$ with high probability for some constant $T_{\max} > \sqrt{2/\pi}$. 

With $\tau$ restricted to the compact interval $[0, T_{\max}]$, we expand the norm term inside Eq.~\eqref{eq:AO} yields:
\begin{equation*}
\frac{1}{n}\|\epsilon - \|v\|_2 h\|_2^2
= \frac{1}{n}\|\epsilon\|_2^2 + \frac{\|v\|_2^2}{n}\|h\|_2^2 - \frac{2\|v\|_2}{n}\epsilon^\top h.
\end{equation*}
By the Strong Law of Large Numbers \citep{ash2000probability,kolmogorov2018foundations}, $\frac{1}{n}\|\epsilon\|_2^2 \xrightarrow{a.s.} \sigma^2$, $\frac{1}{n}\|h\|_2^2 \xrightarrow{a.s.} 1$, and $\frac{1}{n}\epsilon^\top h \xrightarrow{a.s.} 0$. Therefore, uniformly over the compact domains $\mathcal{C}_{R_n}$ and $[0, T_{\max}]$:
\begin{equation*}
\sup_{v \in \mathcal{C}_{R_n}} \Bigl| \frac{1}{n}\|\epsilon - \|v\|_2 h\|_2^2 - (\sigma^2 + \|v\|_2^2) \Bigr| \xrightarrow{a.s.} 0.
\end{equation*}
Because $\tau$ is uniformly bounded, this almost sure convergence maps the AO to its deterministic limit:
\begin{equation}
\label{eq:AO_det}
\min_{v \in \mathcal{C}_{R_n}} \max_{\tau \in [0, T_{\max}]}
\Bigl\{ \|v\|_1 + \tau\sqrt{n}\,\sqrt{\sigma^2 + \|v\|_2^2} - \tau g^\top v \Bigr\}.
\end{equation}

For any $x \ge 0$, the AM-GM inequality provides the variational representation $\sqrt{x} = \min_{b > 0} \frac{1}{2}\bigl(b + \frac{x}{b}\bigr)$. Applying this with $x = \sigma^2 + \|v\|_2^2$ inside Eq.~\eqref{eq:AO_det} yields:
\begin{equation*}
\min_{v \in \mathcal{C}_{R_n}} \max_{\tau \in [0, T_{\max}]} \min_{b > 0}
\Bigl\{ \|v\|_1 + \frac{\tau\sqrt{n}}{2}\Bigl(b + \frac{\sigma^2 + \|v\|_2^2}{b}\Bigr) - \tau g^\top v \Bigr\}.
\end{equation*}
The mapping $(v,b) \mapsto \|v\|_2^2/b$ is the perspective function of the squared norm, making the objective jointly convex in $(v,b)$. Because the objective is jointly convex in $(v,b)$ and concave (linear) in $\tau$, Sion's minimax theorem \citep{sion1958general} allows us to exchange $\max_\tau$ with $\min_{v,b}$ over the compact domains:
\begin{equation*}
\max_{\tau \in [0, T_{\max}]} \min_{b > 0} \min_{v \in \mathcal{C}_{R_n}} \mathcal{L}(v, b, \tau),
\end{equation*}
where the Lagrangian is defined as:
\begin{equation}
\label{eq:lagrangian}
\mathcal{L}(v, b, \tau) = \sum_{i=1}^p \Bigl( |v_i| + \frac{\tau\sqrt{n}}{2b} v_i^2 - \tau g_i v_i \Bigr) + \frac{\tau\sqrt{n}}{2}\Bigl(b + \frac{\sigma^2}{b}\Bigr).
\end{equation}
The inner minimization over $v$ separates coordinate-wise. We complete the square for each $v_i$:
\begin{equation*}
|v_i| + \frac{\tau\sqrt{n}}{2b} v_i^2 - \tau g_i v_i = \frac{\tau\sqrt{n}}{2b} \left[ \Bigl(v_i - \frac{b}{\sqrt{n}} g_i\Bigr)^2 + \frac{2b}{\tau\sqrt{n}} |v_i| - \frac{b^2}{n} g_i^2 \right].
\end{equation*}
The minimizer is given by the proximal operator of the $\ell_1$-norm. Using the homogeneity of the soft-thresholding operator $\mathcal{S}(cx; c\lambda) = c\mathcal{S}(x;\lambda)$, we obtain:
\begin{equation}
\label{eq:vstar}
v_i^* = \mathcal{S}\Bigl(\frac{b}{\sqrt{n}} g_i; \frac{b}{\tau\sqrt{n}}\Bigr) = \frac{b}{\sqrt{n}} \,\mathcal{S}\bigl(g_i; \kappa\bigr), \qquad \text{where} \quad \kappa = \frac{1}{\tau}.
\end{equation}

Let $\alpha_n^2 = \|\hat{v}_{R_n}\|_2^2$. By substituting Eq.~\eqref{eq:vstar} and applying the Law of Large Numbers, the deterministic limit $\alpha^2 = \lim_{n \to \infty} \alpha_n^2$ satisfies:
\begin{equation}
\label{eq:alpha}
\alpha^2 = b^2 \frac{p}{n} \, \mathbb{E}[\mathcal{S}^2(g;\kappa)].
\end{equation}

The optimality condition with respect to $b$ is obtained by setting $\frac{\partial \mathcal{L}}{\partial b} = 0$. Differentiating Eq.~\eqref{eq:lagrangian}:
\begin{equation*}
\frac{\partial \mathcal{L}}{\partial b} = -\frac{\tau\sqrt{n}}{2b^2} \sum_{i=1}^p v_i^2 + \frac{\tau\sqrt{n}}{2} \left(1 - \frac{\sigma^2}{b^2}\right) = \frac{\tau\sqrt{n}}{2} \left( 1 - \frac{\sigma^2 + \|v\|_2^2}{b^2} \right) = 0.
\end{equation*}
In the deterministic limit, $\|v^*\|_2^2 \to \alpha^2$, yielding $b^* = \sqrt{\sigma^2 + \alpha^2}$.

To resolve the dual variable $\tau$ (and thus $\kappa$), we apply Danskin's envelope theorem \citep{danskin2012theory} to the outer maximization. At the saddle point, the partial derivative of $\mathcal{L}$ with respect to $\tau$ vanishes. Differentiating Eq.~\eqref{eq:lagrangian}:
\begin{equation*}
\frac{\partial \mathcal{L}}{\partial \tau} = \sum_{i=1}^p \left( \frac{\sqrt{n}}{2b} v_i^2 - g_i v_i \right) + \frac{\sqrt{n}}{2}\Bigl(b + \frac{\sigma^2}{b}\Bigr) = \frac{\sqrt{n}}{2}\left( b + \frac{\sigma^2 + \|v\|_2^2}{b} \right) - g^\top v.
\end{equation*}
Evaluating at the optimum $(v^*, b^*)$, we substitute $b^* = \sqrt{\sigma^2 + \|v^*\|_2^2}$ so the term in parentheses becomes $2b^*$:
\begin{equation*}
\frac{\partial \mathcal{L}}{\partial \tau}\bigg|_{(v^*, b^*)} = \sqrt{n}b^* - g^\top v^* = 0.
\end{equation*}
Taking the expectation in the deterministic limit gives $\sqrt{n}b^* - \mathbb{E}[g^\top v^*] = 0$. Substituting $v_i^* = \frac{b^*}{\sqrt{n}}\mathcal{S}(g_i; \kappa)$:
\begin{equation*}
\sqrt{n}b^* - \frac{b^*}{\sqrt{n}} \sum_{i=1}^p \mathbb{E}[g_i \mathcal{S}(g_i;\kappa)] = 0 \implies 1 - \frac{p}{n}\mathbb{E}[g \mathcal{S}(g;\kappa)] = 0.
\end{equation*}
Applying Stein's lemma \citep{stein1981estimation}, $\mathbb{E}[g\,\mathcal{S}(g;\kappa)] = \mathbb{P}(|g| > \kappa)$, we obtain the calibration equation:
\begin{equation}
\label{eq:calib}
\frac{p}{n} \, \mathbb{P}(|g| > \kappa) = 1.
\end{equation}

Since $p/n \to \infty$, Eq.~\eqref{eq:calib} requires $\kappa \to \infty$. Using the standard Gaussian tail expansion \citep[\textcolor{red}{Formula 7.1.23}]{abramowitz1966handbook}, $\mathbb{P}(|g| > \kappa) = \frac{2}{\kappa\sqrt{2\pi}} e^{-\kappa^2/2}\bigl(1+o(1)\bigr)$. Taking logarithms yields:
\begin{equation*}
\label{eq:kappa}
\kappa^2 = 2 \log(p/n)\,(1+o(1)).
\end{equation*}

We evaluate the exact truncated second moment:
\begin{equation*}
\label{eq:moment_exact}
\mathbb{E}[\mathcal{S}^2(g;\kappa)] = 2\bigl[(1+\kappa^2)Q(\kappa) - \kappa \varphi(\kappa)\bigr],
\end{equation*}
where $Q(\kappa) = \int_\kappa^\infty \varphi(x)\,dx$ and $\varphi(x) = \frac{1}{\sqrt{2\pi}}e^{-x^2/2}$ is the standard normal PDF. 

\textbf{Upper bound:} Using the refined tail inequality \citep[\textcolor{red}{Chapter VII}]{feller1968introduction} $Q(\kappa) \le \frac{\varphi(\kappa)}{\kappa} \bigl(1 - \frac{1}{\kappa^2} + \frac{3}{\kappa^4}\bigr)$, we substitute and expand:
\begin{align*}
\mathbb{E}[\mathcal{S}^2(g;\kappa)] 
&\le 2\frac{\varphi(\kappa)}{\kappa} \left[ (1+\kappa^2)\left(1 - \frac{1}{\kappa^2} + \frac{3}{\kappa^4}\right) - \kappa^2 \right] \\
&= 2\frac{\varphi(\kappa)}{\kappa} \left( 1 - \frac{1}{\kappa^2} + \frac{3}{\kappa^4} + \kappa^2 - 1 + \frac{3}{\kappa^2} - \kappa^2 \right) \\
&= 2\frac{\varphi(\kappa)}{\kappa} \left( \frac{2}{\kappa^2} + \frac{3}{\kappa^4} \right) \le C_1 \frac{\varphi(\kappa)}{\kappa^3} \le C_2 \frac{1}{\kappa^2}\mathbb{P}(|g|>\kappa),
\end{align*}
where the last step applies $\frac{\varphi(\kappa)}{\kappa} \asymp \mathbb{P}(|g|>\kappa)$.

\textbf{Lower bound:} Restricting the integration interval to $[\kappa, \kappa+1/\kappa]$ gives:
\begin{equation*}
\mathbb{E}[\mathcal{S}^2(g;\kappa)] \ge 2\int_\kappa^{\kappa+1/\kappa} (x-\kappa)^2 \varphi(x)\,dx = 2\int_0^{1/\kappa} t^2 \varphi(\kappa+t)\,dt.
\end{equation*}
Because $\varphi$ is monotonically decreasing for $x>0$:
\begin{equation*}
\varphi(\kappa+t) \ge \varphi(\kappa+1/\kappa) = \varphi(\kappa)\exp\left(-\frac{1}{2}\Bigl(\frac{2\kappa}{\kappa} + \frac{1}{\kappa^2}\Bigr)\right) = \varphi(\kappa)\exp\left(-1 - \frac{1}{2\kappa^2}\right).
\end{equation*}
For $\kappa \ge 1$, $\exp\bigl(-1 - \frac{1}{2\kappa^2}\bigr) \ge e^{-3/2}$. Let $c_1 = e^{-3/2}$. Then:
\begin{equation*}
\mathbb{E}[\mathcal{S}^2(g;\kappa)] \ge 2 c_1 \varphi(\kappa) \int_0^{1/\kappa} t^2 dt = 2 c_1 \varphi(\kappa) \left[ \frac{t^3}{3} \right]_0^{1/\kappa} = \frac{2 c_1}{3\kappa^3}\varphi(\kappa) \ge c_2 \frac{1}{\kappa^2}\mathbb{P}(|g|>\kappa).
\end{equation*}
Combining both bounds yields $\mathbb{E}[\mathcal{S}^2(g;\kappa)] = \Theta\left(\frac{1}{\kappa^2}\mathbb{P}(|g|>\kappa)\right)$. Substituting $\mathbb{P}(|g| > \kappa) = n/p$ from Eq.~\eqref{eq:calib}, we get $\mathbb{E}[\mathcal{S}^2(g;\kappa)] = \Theta\left(\frac{n}{p\kappa^2}\right)$. 

Inserting this into Eq.~\eqref{eq:alpha} alongside $b^2 = \sigma^2 + \alpha^2$ yields:
\begin{equation*}
\alpha^2 = (\sigma^2 + \alpha^2) \frac{p}{n} \Theta\left(\frac{n}{p\kappa^2}\right) = (\sigma^2 + \alpha^2) \Theta\left(\frac{1}{\kappa^2}\right) \implies \alpha^2 = \Theta\left(\frac{\sigma^2}{\kappa^2}\right) = \Theta\left(\frac{\sigma^2}{\log(p/n)}\right).
\end{equation*}

By the CGMT uniform convergence over $\mathcal{C}_{R_n}$, $\|\hat{v}_{R_n}\|_2^2 \xrightarrow{\mathbb{P}} \alpha^2$. Since $\|\hat{v}_{R_n}\|_2^2 \le R_n^2 = \mathcal{O}(n)$, the sequence is deterministically bounded by a polynomial in $n$, securing uniform integrability. Hence $\mathbb{E}\|\hat{v}_{R_n}\|_2^2 \to \alpha^2$.

We decompose the expected risk of the unconstrained estimator:
\begin{equation*}
\mathbb{E}\|\hat{v}_n\|_2^2 = \mathbb{E}\|\hat{v}_{R_n}\|_2^2 + \mathbb{E}\bigl[ (\|\hat{v}_n\|_2^2 - \|\hat{v}_{R_n}\|_2^2) \mathbf{1}_{\{\|\hat{v}_n\|_2 > R_n\}} \bigr].
\end{equation*}
On the event $\{\|\hat{v}_n\|_2 > R_n\}$, we have $\|\hat{v}_n\|_2^2 - \|\hat{v}_{R_n}\|_2^2 \le \|\hat{v}_n\|_2^2 \le \|\tilde{v}\|_2^2$. By the Cauchy-Schwarz inequality:
\begin{equation*}
\mathbb{E}\bigl[ (\|\hat{v}_n\|_2^2 - \|\hat{v}_{R_n}\|_2^2) \mathbf{1}_{\{\|\hat{v}_n\|_2 > R_n\}} \bigr]
\le \sqrt{\mathbb{E}\|\tilde{v}\|_2^4} \; \sqrt{\mathbb{P}(\|\hat{v}_n\|_2 > R_n)}.
\end{equation*}
The fourth moment is expanded using the independence of $Z$ and $\epsilon$:
\begin{equation*}
\mathbb{E}\|\tilde{v}\|_2^4 \le p^2 \mathbb{E}\bigl[\|Z^\dagger\|_2^4 \|\epsilon\|_2^4\bigr] = p^2 \mathbb{E}\|Z^\dagger\|_2^4 \mathbb{E}\|\epsilon\|_2^4.
\end{equation*}
Since $\epsilon \sim \mathcal{N}(0, \sigma^2 I_n)$, $\mathbb{E}\|\epsilon\|_2^4 = \sigma^4(n^2 + 2n) = \mathcal{O}(n^2)$. From standard inverse Wishart moments \citep{von1988moments,muirhead2009aspects}, $\mathbb{E}\|Z^\dagger\|_2^4 = \mathcal{O}(p^{-2})$. Thus, $\mathbb{E}\|\tilde{v}\|_2^4 = \mathcal{O}(n^2)$.

Meanwhile, we strictly bound the deviation probability $\mathbb{P}(\|\hat{v}_n\|_2 > R_n) \le \mathbb{P}(\|\tilde{v}\|_2 > C_0\sqrt{n})$. Since $\|\tilde{v}\|_2 \le \|\epsilon\|_2 / \sigma_{\min}(Z)$, we can decompose this probability via the set inclusion for any constants satisfying $C_0 = c_1 / c_2$:
\begin{equation*}
\bigl\{ \|\tilde{v}\|_2 > C_0 \sqrt{n} \bigr\} \subseteq \bigl\{ \|\epsilon\|_2 > c_1 \sqrt{n} \bigr\} \cup \bigl\{ \sigma_{\min}(Z) < c_2 \sqrt{p} \bigr\}.
\end{equation*}

By standard chi-square concentration \citep[\textcolor{red}{Lemma 1}]{laurent2000adaptive}, choosing $c_1 > \sigma$ sufficiently large ensures that the probability $\mathbb{P}(\|\epsilon\|_2 > c_1\sqrt{n}) \le e^{-cn}$ decays exponentially in $n$. 

Moreover, for the Gaussian matrix $Z$, singular value bounds \citep[\textcolor{red}{Theorem 5.32}]{vershynin2012introduction} state that for any $t > 0$,
\begin{equation*}
\mathbb{P}\bigl( \sigma_{\min}(Z) \le \sqrt{p} - \sqrt{n} - t \bigr) \le e^{-t^2/2}.
\end{equation*}
Since $p/n \to \infty$, for sufficiently large $n$ and $p$, we have $\sqrt{n} \le \frac{1}{4}\sqrt{p}$. Choosing $t = \frac{1}{4}\sqrt{p}$ and setting $c_2 = 1/2$ yields:
\begin{equation*}
\mathbb{P}\Bigl( \sigma_{\min}(Z) \le \frac{1}{2}\sqrt{p} \Bigr) \le e^{-p/32}.
\end{equation*}

By appropriately choosing $c_1$ and defining $C_0 = 2c1$, the union bound guarantees that the combined deviation probability satisfies:
\begin{equation*}
\mathbb{P}(\|\tilde{v}\|_2 > C_0\sqrt{n}) \le e^{-cn} + e^{-p/32}.
\end{equation*}
Because $p > n$, this probability decays exponentially in $n$. Thus, the residual expectation vanishes, yielding:
\begin{equation*}
\mathbb{E}\|\hat{v}_n\|_2^2 \to \alpha^2 = \Theta\!\left( \frac{\sigma^2}{\log(p/n)} \right).
\end{equation*}
Returning to the global notation, $\hat{\beta}_\infty = \hat{v}_n$ and the expected excess risk is exactly $\mathbb{E}\|\hat{\beta}_\infty\|_2^2$. This completes the proof.
\end{proof}

\section{Missing Proof of Theorem \ref{thm:risk_spiked_final}}
\label{app:proof_thm_risk_spiked_final}

\begin{proof}
Let the generative model be $Y = Z\Sigma^{1/2}\beta^* + \sqrt{n}\epsilon$, where, to match standard CGMT literature notation, we use $Z \in \mathbb{R}^{n \times p}$ with i.i.d.\ $\mathcal{N}(0, 1)$ entries in place of the design matrix $\Phi$, and $\epsilon \sim \mathcal{N}(0, \sigma^2 I_n)$. Let $\nu = v - \beta^*$ denote the estimation error vector. The estimator $\hat{\nu}_n = \hat{\beta}_\infty - \beta^*$ solves:
\begin{equation*}
    \hat{\nu}_n \in \arg\min_{\nu \in \mathbb{R}^p} \|\nu + \beta^*\|_1 \quad \text{s.t.} \quad Z\Sigma^{1/2}\nu = \sqrt{n}\epsilon.
\end{equation*}

Define a feasible solution $\tilde{\nu} = \Sigma^{-1/2} Z^\dagger (\sqrt{n}\epsilon)$. As $p/n \to \infty$, the Bai-Yin law gives $\sigma_{\min}(Z) \sim \sqrt{p}$. Thus, the operator norm is bounded by:
\begin{equation*}
    \|Z^\dagger\|_2 = \frac{1}{\sigma_{\min}(Z)} = \mathcal{O}_P(p^{-1/2}).
\end{equation*}
Using $\|\sqrt{n}\epsilon\|_2 = \mathcal{O}_P(n)$ and $\|\Sigma^{-1/2}\|_2 \le \lambda_{\text{tail}}^{-1/2}$:
\begin{equation*}
    \|\tilde{\nu}\|_2 \le \|\Sigma^{-1/2}\|_2 \|Z^\dagger\|_2 \|\sqrt{n}\epsilon\|_2 = \mathcal{O}_P(n/\sqrt{p}).
\end{equation*}
By the optimality of $\hat{\nu}_n$ and $\|\cdot\|_1 \le \sqrt{p}\|\cdot\|_2$:
\begin{equation*}
    \|\hat{\nu}_n\|_2 \le \|\hat{\nu}_n\|_1 \le \|\tilde{\nu} + \beta^*\|_1 + \|\beta^*\|_1 \le \sqrt{p}\|\tilde{\nu}\|_2 + 2\|\beta^*\|_1.
\end{equation*}
Since $\beta^*$ is $s$-sparse with bounded magnitude, $2\|\beta^*\|_1 = \mathcal{O}(1)$. Consequently:
\begin{equation*}
    \|\hat{\nu}_n\|_2 \le \sqrt{p}\mathcal{O}_P(n/\sqrt{p}) + \mathcal{O}(1) = \mathcal{O}_P(n).
\end{equation*}
There exists a sequence $R_n = c n$ such that $\mathbb{P}(\|\hat{\nu}_n\|_2 > R_n) \to 0$. We restrict the domain to $\mathcal{C}_{R_n} = \{\nu \in \mathbb{R}^p: \|\nu\|_2 \le R_n\}$. Let $\hat{\nu}_{R_n}$ denote the constrained estimator. Thus $\mathbb{P}(\hat{\nu}_n = \hat{\nu}_{R_n}) \to 1$.

Introducing $\mu \in \mathbb{R}^n$, the primary optimization is:
\begin{equation*}
    \min_{\nu \in \mathcal{C}_{R_n}} \max_{\mu \in \mathbb{R}^n} \left\{ \frac{1}{n}\|\nu + \beta^*\|_1 + \frac{1}{n} \mu^\top (\sqrt{n}\epsilon - Z\Sigma^{1/2}\nu) \right\}.
\end{equation*}
Applying the CGMT \citep{gordon2006milman,thrampoulidis2014gaussian,thrampoulidis2015regularized} over $\mathcal{C}_{R_n}$ with independent standard Gaussian vectors $g \sim \mathcal{N}(0, I_p)$ and $h \sim \mathcal{N}(0, I_n)$, and maximizing over the direction of $\mu$, yields $\tilde{\tau} = \|\mu\|_2 / \sqrt{n} \ge 0$:
\begin{equation*}
    \min_{\nu \in \mathcal{C}_{R_n}} \max_{\tilde{\tau} \ge 0} \left\{ \frac{1}{n}\|\nu + \beta^*\|_1 + \frac{\tilde{\tau}}{\sqrt{n}} \big\| \sqrt{n}\epsilon - \|\Sigma^{1/2}\nu\|_2 h \big\|_2 - \frac{\tilde{\tau}}{\sqrt{n}} g^\top \Sigma^{1/2} \nu \right\}.
\end{equation*}
Since $\nu$ is confined to the compact set $\mathcal{C}_{R_n}$, Sion's Minimax Theorem \citep{sion1958general} allows the commutation of the min and max operators. 

Define $\zeta = \frac{1}{\sqrt{n}}\|\Sigma^{1/2}\nu\|_2$. Evaluating the objective at the feasible point $\tilde{\nu}$ yields a value of order $\mathcal{O}(1)$. Hence the optimal value is upper bounded by a constant. For any fixed $\tilde{\tau} \ge 0$, the term
\begin{equation*}
\frac{\tilde{\tau}}{\sqrt{n}}\|\sqrt{n}\epsilon - \sqrt{n}\zeta h\|_2
\end{equation*}
satisfies
\begin{equation*}
\frac{1}{n}\|\sqrt{n}\epsilon - \sqrt{n}\zeta h\|_2^2
= \frac{1}{n}\sum_{j=1}^n (\sqrt{n}\epsilon_j - \sqrt{n}\zeta h_j)^2.
\end{equation*}
Since $\epsilon_j \sim \mathcal{N}(0,\sigma^2)$ and $h_j \sim \mathcal{N}(0,1)$ are independent, we have
\begin{equation*}
\mathbb{E}(\sqrt{n}\epsilon_j - \sqrt{n}\zeta h_j)^2 = n(\sigma^2 + \zeta^2).
\end{equation*}
Thus,
\begin{equation*}
\frac{1}{n}\|\sqrt{n}\epsilon - \sqrt{n}\zeta h\|_2^2 \to \sigma^2 + \zeta^2.
\end{equation*}
Hence,
\begin{equation*}
\frac{\tilde{\tau}}{\sqrt{n}}\|\sqrt{n}\epsilon - \sqrt{n}\zeta h\|_2
= \tilde{\tau}\sqrt{\sigma^2 + \zeta^2} + o_P(1).
\end{equation*}
As $\zeta \to \infty$, we have $\sqrt{\sigma^2 + \zeta^2} \sim \zeta$, hence the objective diverges linearly in $\zeta$. Therefore, any optimizer $\zeta^*$ must lie in a bounded interval $[0,C]$ uniformly in $n$. This justifies the application of the Law of Large Numbers \citep{ash2000probability,kolmogorov2018foundations}:
\begin{equation*}
\frac{1}{n}\|\sqrt{n}\epsilon - \sqrt{n}\zeta h\|_2^2 \to \sigma^2 + \zeta^2.
\end{equation*}

Using the variational upper bound (a rearrangement of the AM-GM inequality) $\sqrt{X} = \min_{b>0} \left( \frac{X}{2b} + \frac{b}{2} \right)$ for $X = \sigma^2 + \zeta^2$, the objective becomes:
\begin{equation*}
\max_{\tilde{\tau} \ge 0} \min_{b>0} \left\{ \frac{\tilde{\tau} b}{2} + \frac{\tilde{\tau} \sigma^2}{2b} + \frac{1}{n} \sum_{i=1}^p \min_{\nu_i} \left[ |\nu_i + \beta_i^*| + \frac{\tilde{\tau}\lambda_i}{2b} \nu_i^2 - \tilde{\tau}\sqrt{n\lambda_i} g_i \nu_i \right] \right\}.
\end{equation*}
Let $x_i = \nu_i + \beta_i^*$. Then:
\begin{align*}
&|\nu_i + \beta_i^*| + \frac{\tilde{\tau}\lambda_i}{2b} \nu_i^2 - \tilde{\tau}\sqrt{n\lambda_i} g_i \nu_i \\
&= |x_i| + \frac{\tilde{\tau}\lambda_i}{2b}(x_i - \beta_i^*)^2 - \tilde{\tau}\sqrt{n\lambda_i} g_i (x_i - \beta_i^*) \\
&= |x_i| + \frac{\tilde{\tau}\lambda_i}{2b} \left( x_i^2 - 2x_i\beta_i^* + (\beta_i^*)^2 \right)
- \tilde{\tau}\sqrt{n\lambda_i} g_i x_i + \tilde{\tau}\sqrt{n\lambda_i} g_i \beta_i^* \\
&= |x_i| + \frac{\tilde{\tau}\lambda_i}{2b} x_i^2
- \left( \frac{\tilde{\tau}\lambda_i}{b}\beta_i^* + \tilde{\tau}\sqrt{n\lambda_i} g_i \right) x_i + C \\
&= |x_i| + \frac{\tilde{\tau}\lambda_i}{2b}
\left[
x_i^2 - 2x_i\left( \beta_i^* + \frac{b\sqrt{n}}{\sqrt{\lambda_i}} g_i \right)
\right] + C \\
&= |x_i| + \frac{\tilde{\tau}\lambda_i}{2b}
\left(
x_i - \left( \beta_i^* + \frac{b\sqrt{n}}{\sqrt{\lambda_i}} g_i \right)
\right)^2 + C'.
\end{align*}
Thus,
\begin{equation*}
x_i^* = \mathcal{S}\left( \beta_i^* + \frac{b\sqrt{n}}{\sqrt{\lambda_i}} g_i ; \frac{b}{\tilde{\tau}\lambda_i} \right).
\end{equation*}
Define
\begin{equation*}
\tau^* = \tilde{\tau}\sqrt{n}, \quad \kappa_i = \frac{1}{\tau^*\sqrt{\lambda_i}}, \quad c_i = \frac{\sqrt{\lambda_i}}{b\sqrt{n}}\beta_i^*.
\end{equation*}
Using the homogeneity property $\mathcal{S}(cA;cB) = c\mathcal{S}(A;B)$, we obtain
\begin{equation*}
\nu_i^* = \frac{b\sqrt{n}}{\sqrt{\lambda_i}} \left[ \mathcal{S}(g_i + c_i; \kappa_i) - c_i \right].
\end{equation*}

For $i \le k^*$, the expected error is:
\begin{equation*}
    \mathcal{E}_{\mathrm{head}} = b^2 \frac{1}{n} \sum_{i=1}^{k^*} \mathbb{E}_{g_i}\left[ \left( \mathcal{S}(g_i + c_i; \kappa_i) - c_i \right)^2 \right].
\end{equation*}
Let $F(g) = \mathcal{S}(g+c; \kappa) - c$. Evaluating the three regions of $\mathcal{S}$:
If $g+c > \kappa$: $F(g) = g-\kappa$. Since $g > \kappa-c \implies -c < g-\kappa \le g+\kappa$. Thus $|F(g)| \le |g| + \kappa$.
If $g+c < -\kappa$: $F(g) = g+\kappa$. Since $g < -\kappa-c \implies -g-\kappa > c$. Thus $|F(g)| = -g-\kappa \le |g| + \kappa$.
If $|g+c| \le \kappa$: $F(g) = -c$. Since $-\kappa \le g+c \le \kappa \implies -\kappa-g \le c \le \kappa-g$, meaning $|c| \le |g| + \kappa$. Thus $|F(g)| \le |g| + \kappa$.
Therefore, $\left( \mathcal{S}(g_i + c_i; \kappa_i) - c_i \right)^2 \le (|g_i| + \kappa_i)^2 \le 2g_i^2 + 2\kappa_i^2$. Taking the expectation over $g_i \sim \mathcal{N}(0,1)$:
\begin{equation*}
    \mathcal{E}_{\mathrm{head}} \le 2b^2 \frac{1}{n} \sum_{i=1}^{k^*} (1 + \kappa_i^2) = \mathcal{O}\left( \sigma^2 \frac{k^*}{n} \right).
\end{equation*}

For $i > k^*$, $\beta_i^* = 0 \implies c_i = 0$, and $\kappa_i = \kappa_T$. The tail risk is:
\begin{equation*}
    \mathcal{E}_{\mathrm{tail}} = b^2 \frac{r_2}{n} \mathbb{E}[\mathcal{S}^2(g; \kappa_T)].
\end{equation*}
Let $P_T = \mathbb{P}(|g| > \kappa_T) = 2Q(\kappa_T)$. Differentiating the objective with respect to $b$ gives $\frac{1}{n}\sum_{i=1}^{k^*} \mathbb{P}(|g_i+c_i|>\kappa_i) + \frac{r_2}{n} P_T = 1$. Since $k^* = o(n)$, $\frac{r_2}{n}P_T \to 1$.
Using $P_T \sim \sqrt{\frac{2}{\pi}} \frac{1}{\kappa_T} e^{-\kappa_T^2/2}$ yields $\kappa_T^2 \sim 2 \log(r_2/n)$.
Integration by parts for the truncated second moment gives:
\begin{align*}
    \mathbb{E}[\mathcal{S}^2] &= 2 \int_{\kappa_T}^\infty (x-\kappa_T)^2 \varphi(x) dx \\
    &= 2 \left( \int_{\kappa_T}^\infty x^2\varphi(x) dx - 2\kappa_T \int_{\kappa_T}^\infty x\varphi(x) dx + \kappa_T^2 \int_{\kappa_T}^\infty \varphi(x) dx \right) \\
    &= 2 \left[ (\kappa_T\varphi(\kappa_T) + Q(\kappa_T)) - 2\kappa_T\varphi(\kappa_T) + \kappa_T^2 Q(\kappa_T) \right] \\
    &= 2 \left[ (1+\kappa_T^2)Q(\kappa_T) - \kappa_T\varphi(\kappa_T) \right].
\end{align*}
Using $Q(\kappa_T) \le \varphi(\kappa_T)\left(\frac{1}{\kappa_T} - \frac{1}{\kappa_T^3} + \frac{3}{\kappa_T^5}\right)$:
\begin{align*}
    \frac{\mathbb{E}[\mathcal{S}^2]}{2\varphi(\kappa_T)} &\le (1+\kappa_T^2)\left(\frac{1}{\kappa_T} - \frac{1}{\kappa_T^3} + \frac{3}{\kappa_T^5}\right) - \kappa_T \\
    &= \left( \frac{1}{\kappa_T} - \frac{1}{\kappa_T^3} + \frac{3}{\kappa_T^5} \right) + \left( \kappa_T - \frac{1}{\kappa_T} + \frac{3}{\kappa_T^3} \right) - \kappa_T \\
    &= \frac{2}{\kappa_T^3} + \frac{3}{\kappa_T^5}.
\end{align*}
For $\kappa_T \ge 1$, $\frac{3}{\kappa_T^5} \le \frac{3}{\kappa_T^3}$, yielding $\frac{\mathbb{E}[\mathcal{S}^2]}{2\varphi(\kappa_T)} \le \frac{5}{\kappa_T^3} = \frac{5}{\kappa_T^2} \frac{1}{\kappa_T}$.
Using $Q(\kappa_T) \ge \varphi(\kappa_T)\left(\frac{1}{\kappa_T} - \frac{1}{\kappa_T^3} + \frac{3}{\kappa_T^5} - \frac{15}{\kappa_T^7}\right)$:
\begin{align*}
    \frac{\mathbb{E}[\mathcal{S}^2]}{2\varphi(\kappa_T)} &\ge (1+\kappa_T^2)\left(\frac{1}{\kappa_T} - \frac{1}{\kappa_T^3} + \frac{3}{\kappa_T^5} - \frac{15}{\kappa_T^7}\right) - \kappa_T \\
    &= \left( \frac{1}{\kappa_T} - \frac{1}{\kappa_T^3} + \frac{3}{\kappa_T^5} - \frac{15}{\kappa_T^7} \right) + \left( \kappa_T - \frac{1}{\kappa_T} + \frac{3}{\kappa_T^3} - \frac{15}{\kappa_T^5} \right) - \kappa_T \\
    &= \frac{2}{\kappa_T^3} - \frac{12}{\kappa_T^5} - \frac{15}{\kappa_T^7} = \frac{1}{\kappa_T^3} \left( 2 - \frac{12}{\kappa_T^2} - \frac{15}{\kappa_T^4} \right).
\end{align*}
For sufficiently large $\kappa_T$, $2 - \frac{12}{\kappa_T^2} - \frac{15}{\kappa_T^4} \ge 1$. Thus, $\frac{\mathbb{E}[\mathcal{S}^2]}{2\varphi(\kappa_T)} \ge \frac{1}{\kappa_T^3} = \frac{1}{\kappa_T^2} \frac{1}{\kappa_T}$.
Using $\frac{\varphi(\kappa_T)}{\kappa_T} \sim Q(\kappa_T) = P_T / 2$ yields $\mathbb{E}[\mathcal{S}^2] = \Theta\left( \frac{1}{\kappa_T^2} P_T \right)$.
Substituting this into $\mathcal{E}_{\mathrm{tail}}$:
\begin{equation*}
    \mathcal{E}_{\mathrm{tail}} = b^2 \frac{r_2}{n} \Theta\left( \frac{1}{\kappa_T^2} P_T \right) = \Theta\left( \frac{b^2}{\kappa_T^2} \right) \left[ \frac{r_2}{n} P_T \right].
\end{equation*}
Applying $\frac{r_2}{n}P_T \to 1$, $\kappa_T^2 \sim 2\log(r_2/n)$, and $(b^*)^2 = \Theta(\sigma^2)$:
\begin{equation*}
    \mathcal{E}_{\mathrm{tail}} = \Theta\left( \frac{\sigma^2}{\log(r_2/n)} \right).
\end{equation*}

The expected risk is decomposed using the domain indicator function:
\begin{equation*}
    \mathbb{E}[\mathcal{E}(\hat{\beta}_\infty)] = \frac{1}{n}\mathbb{E}\|\Sigma^{1/2}\hat{\nu}_{R_n}\|_2^2 + \frac{1}{n}\mathbb{E}\left[ \|\Sigma^{1/2}\hat{\nu}_n\|_2^2 \mathbf{1}_{\{\|\hat{\nu}_n\|_2 > R_n\}} \right].
\end{equation*}
The first term converges to $\mathcal{E}_{\mathrm{head}} + \mathcal{E}_{\mathrm{tail}}$. By the Cauchy-Schwarz inequality:
\begin{equation*}
    \frac{1}{n}\mathbb{E}\left[ \|\Sigma^{1/2}\hat{\nu}_n\|_2^2 \mathbf{1}_{\{\|\hat{\nu}_n\|_2 > R_n\}} \right] \le \frac{1}{n} \sqrt{\mathbb{E}\|\Sigma^{1/2}\hat{\nu}_n\|_2^4} \sqrt{\mathbb{P}(\|\hat{\nu}_n\|_2 > R_n)}.
\end{equation*}
For $p/n \to \infty$, the polynomial moments of the Inverse-Wishart matrix $(Z Z^\top)^{-1}$ are finite, yielding $\mathbb{E}\|\Sigma^{1/2}\hat{\nu}_n\|_2^4 \le \mathcal{O}(n^C)$. 
The deviation probability satisfies $\mathbb{P}(\|\hat{\nu}_n\|_2 > R_n) \le e^{-c n}$. The residual term decays to 0. 
\begin{equation*}
    \mathbb{E}[\mathcal{E}(\hat{\beta}_\infty)] = \mathcal{O}\left( \sigma^2 \frac{k^*}{n} \right) + \Theta\left( \frac{\sigma^2}{\log(r_2/n)} \right).
\end{equation*}
This concludes the proof.
\end{proof}

\section{Missing Proof of Theorem~\ref{thm:early_stopping}}
\label{app:proof_thm_early_stopping}

\begin{proof}
We analyze the finite-time dynamic path of the continuous $\ell_2$-Boosting flow and derive the basic inequality at the analytical stopping time $t^*$.

We first prove that the maximum absolute empirical correlation $\rho(t) = \|g(t)\|_\infty$ is monotonically non-increasing. The negative gradient of the empirical risk is $g(t) = \frac{1}{n}\Phi^\top(Y - \Phi\beta(t))$. By standard differential inclusion theory, the trajectory $\beta(t)$ governed by $\dot{\beta}(t) \in \partial\|g(t)\|_\infty$ is absolutely continuous. Its temporal derivative evaluates to:
\begin{equation*}
\dot{g}(t) = -\frac{1}{n}\Phi^\top\Phi\dot{\beta}(t).
\end{equation*}
Let $s(t) \in \partial\|g(t)\|_\infty$ be a subgradient. By definition, we may choose $s(t) = \dot{\beta}(t)$. Utilizing Danskin's envelope theorem \citep{danskin2012theory} for the subdifferential of the $\ell_\infty$-norm, the time derivative of $\rho(t)$ resolves to the inner product:
\begin{equation*}
\frac{d}{dt}\rho(t) = \frac{d}{dt}\|g(t)\|_\infty = \langle s(t), \dot{g}(t)\rangle = \langle \dot{\beta}(t), \dot{g}(t)\rangle.
\end{equation*}
Substituting $\dot{g}(t)$ into the derivative, we obtain:
\begin{equation*}
\frac{d}{dt}\rho(t) = \left\langle \dot{\beta}(t), -\frac{1}{n}\Phi^\top\Phi\dot{\beta}(t) \right\rangle = -\frac{1}{n}\|\Phi\dot{\beta}(t)\|_2^2 \le 0.
\end{equation*}
Thus, $\rho(t)$ is continuous and monotonically non-increasing. Assume the initial correlation exceeds the noise floor, $\rho(0) = \left\|\frac{1}{n}\Phi^\top Y\right\|_\infty > 2\lambda_n$. By the Intermediate Value Theorem and monotonicity, there exists a unique time $t^*$ such that $\rho(t^*) = 2\lambda_n$, which exactly coincides with the stopping time definition $t^* = \inf\{t > 0 : \rho(t) \le 2\lambda_n\}$.

Next, we establish the high-probability bound for the noise correlation vector $\xi = \frac{1}{n}\Phi^\top\epsilon$. Conditional on the random design matrix $\Phi$, each component $\xi_j = \frac{1}{n}\Phi_j^\top\epsilon$ is a linear combination of independent sub-Gaussian random variables $\epsilon$. Thus, conditional on $\Phi$, $\xi_j$ is sub-Gaussian with variance proxy $\frac{\sigma^2}{n^2}\|\Phi_j\|_2^2$. By the column-normalization assumption $\frac{1}{n}\|\Phi_j\|_2^2 \le 1$, this conditional variance proxy is bounded by $\frac{\sigma^2}{n}$.

Consequently, for any $t > 0$, the conditional tail probability satisfies:
\begin{equation*}
\mathbb{P}\bigl(|\xi_j| \ge t \mid \Phi\bigr) \le 2\exp\left(-\frac{n t^2}{2\sigma^2}\right).
\end{equation*}
Setting the threshold $t = \lambda_n = \sigma\sqrt{\frac{2c\log p}{n}}$ for a constant $c > 1$ and applying a union bound conditionally on $\Phi$ yields:
\begin{equation*}
\mathbb{P}\bigl(\|\xi\|_\infty \ge \lambda_n \mid \Phi\bigr) \le \sum_{j=1}^p 2\exp\left(-\frac{n}{2\sigma^2} \cdot \frac{2c\sigma^2\log p}{n}\right) = 2p^{1-c}.
\end{equation*}
The right-hand side is deterministic. Integrating over the distribution of $\Phi$ via the law of total probability (Tonelli's theorem) yields the unconditional bound:
\begin{equation*}
\mathbb{P}\bigl(\|\xi\|_\infty \ge \lambda_n\bigr) = \mathbb{E}_\Phi\Bigl[\mathbb{P}\bigl(\|\xi\|_\infty \ge \lambda_n \mid \Phi\bigr)\Bigr] \le 2p^{1-c}.
\end{equation*}
We condition the remainder of the proof on this high-probability event $\Omega_0 = \{\|\xi\|_\infty \le \lambda_n\}$.

Let $\hat{\beta}^* = \beta(t^*)$ denote the early-stopped estimator. The empirical prediction error is given by:
\begin{equation*}
\frac{1}{n}\|\Phi(\hat{\beta}^* - \beta^*)\|_2^2 = \left\langle \frac{1}{n}\Phi^\top\Phi(\hat{\beta}^* - \beta^*), \hat{\beta}^* - \beta^* \right\rangle.
\end{equation*}
By substituting the generative model $Y = \Phi\beta^* + \epsilon$, we have $\Phi(\hat{\beta}^* - \beta^*) = \Phi\hat{\beta}^* - Y + \epsilon$. Expanding the inner product:
\begin{equation*}
\frac{1}{n}\|\Phi(\hat{\beta}^* - \beta^*)\|_2^2 = \left\langle \frac{1}{n}\Phi^\top(\Phi\hat{\beta}^* - Y), \hat{\beta}^* - \beta^* \right\rangle + \left\langle \frac{1}{n}\Phi^\top\epsilon, \hat{\beta}^* - \beta^* \right\rangle.
\end{equation*}
Notice that the negative gradient at $t^*$ is $g(t^*) = \frac{1}{n}\Phi^\top(Y - \Phi\hat{\beta}^*)$. Hence, the first term inside the inner product is exactly $-g(t^*)$. We bound both terms using Hölder's inequality $\langle a, b \rangle \le \|a\|_\infty \|b\|_1$:
\begin{equation}
\label{eq:holder_bound}
\frac{1}{n}\|\Phi(\hat{\beta}^* - \beta^*)\|_2^2 \le \|-g(t^*)\|_\infty \|\hat{\beta}^* - \beta^*\|_1 + \|\xi\|_\infty \|\hat{\beta}^* - \beta^*\|_1.
\end{equation}
Because the $\ell_\infty$-norm is symmetric, $\|-g(t^*)\|_\infty = \|g(t^*)\|_\infty$. By the established property of the stopping time $t^*$, we have exactly $\|g(t^*)\|_\infty = \rho(t^*) = 2\lambda_n$. Under the event $\Omega_0$, the noise term satisfies $\|\xi\|_\infty \le \lambda_n$. Substituting these bounds into Eq.~\eqref{eq:holder_bound} yields the fundamental basic inequality:
\begin{equation}
\label{eq:basic_inequality}
\frac{1}{n}\|\Phi(\hat{\beta}^* - \beta^*)\|_2^2 \le 2\lambda_n \|\hat{\beta}^* - \beta^*\|_1 + \lambda_n \|\hat{\beta}^* - \beta^*\|_1 = 3\lambda_n \|\hat{\beta}^* - \beta^*\|_1.
\end{equation}

By the triangle inequality, we expand the $\ell_1$-norm on the right-hand side:
\begin{equation*}
\|\hat{\beta}^* - \beta^*\|_1 \le \|\hat{\beta}^*\|_1 + \|\beta^*\|_1.
\end{equation*}
Conditioned on the premise that the algorithmic path norm at the stopping time satisfies $\|\hat{\beta}^*\|_1 \le C\|\beta^*\|_1$ for some universal constant $C \ge 1$, we obtain:
\begin{equation*}
\|\hat{\beta}^* - \beta^*\|_1 \le (C+1)\|\beta^*\|_1.
\end{equation*}
Substituting this back into Eq.~\eqref{eq:basic_inequality}, the empirical prediction error is deterministically controlled by:
\begin{equation*}
\frac{1}{n}\|\Phi(\hat{\beta}^* - \beta^*)\|_2^2 \le 3(C+1)\lambda_n \|\beta^*\|_1.
\end{equation*}
Inserting the definition of the threshold $\lambda_n = \sigma\sqrt{\frac{2c\log p}{n}}$, we attain the minimax optimal rate:
\begin{equation*}
\frac{1}{n}\|\Phi(\hat{\beta}^* - \beta^*)\|_2^2 \le 3(C+1)\|\beta^*\|_1 \sigma\sqrt{\frac{2c\log p}{n}} = \mathcal{O}\left(\|\beta^*\|_1 \sigma\sqrt{\frac{\log p}{n}}\right).
\end{equation*}
This bound holds with probability at least $1 - 2p^{1-c}$, completing the proof.
\end{proof}

\section{Exploratory Extension: Adaptive Trees and XGBoost}
\label{app:xgb_ver}

\begin{figure}[htbp]
    \centering
    \includegraphics[width=0.95\textwidth]{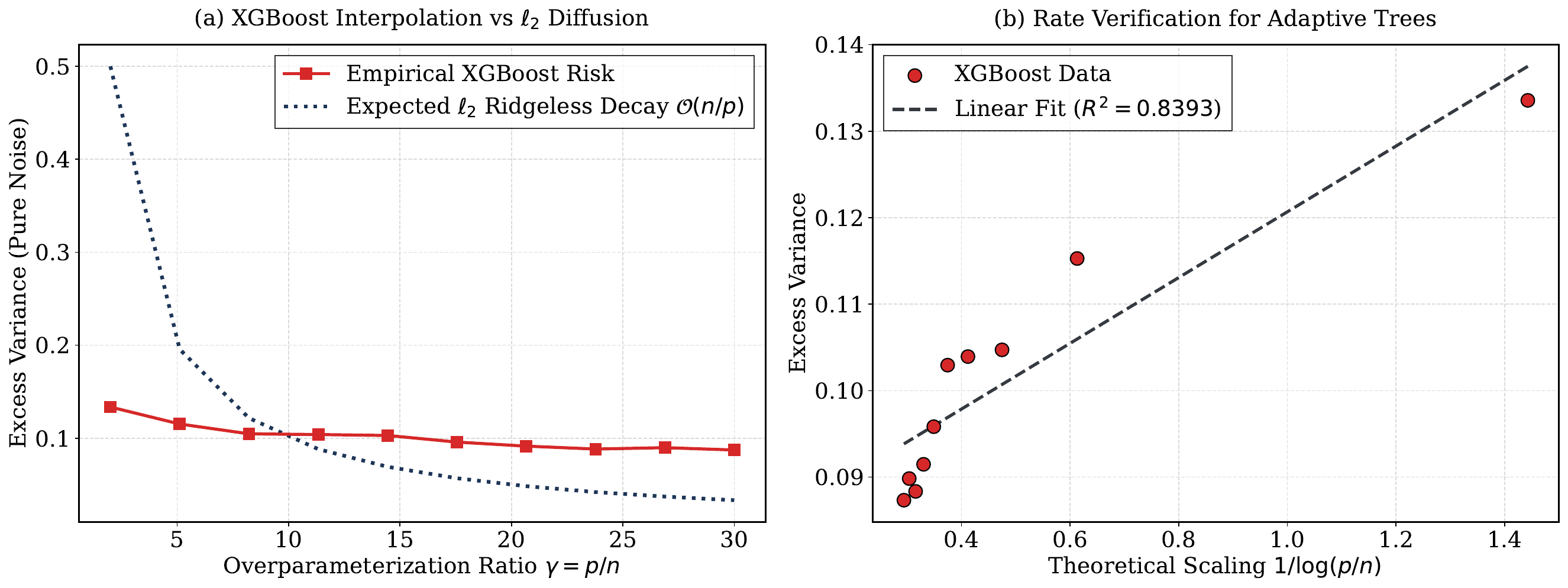} 
    \caption{Empirical evaluation of XGBoost interpolating pure noise ($n=800$). \textbf{(a)} The XGBoost risk trajectory diverges severely from the $\mathcal{O}(n/p)$ theoretical decay of $\ell_2$ interpolation, maintaining an elevated plateau. \textbf{(b)} OLS regression against $1/\log(p/n)$ yields $R^2 = 0.8393$. This strong correlation suggests that the greedy feature generation in adaptive trees inherits the structural noise localization penalty established in our fixed-dictionary analysis.}
    \label{fig:xgboost}
\end{figure}

\paragraph{Experimental Setup}
The risk asymptotics in this paper are mathematically derived for $\ell_2$-Boosting over a fixed dictionary. To explore whether the logarithmic variance penalty persists empirically under adaptive feature generation, we simulated pure noise interpolation ($\beta^* = \mathbf{0}, \sigma=1.0$) using the \texttt{XGBRegressor}. We scaled the system size to $n=800$ and evaluated the overparameterization ratio $\gamma = p/n \in [2, 30]$. To isolate the algorithmic implicit bias from explicit regularization, we utilized a limit-stopping protocol: we deactivated all explicit penalization (\texttt{reg\_alpha=0}, \texttt{reg\_lambda=0}, \texttt{gamma=0}), employed shallow trees (\texttt{max\_depth=2}), set \texttt{learning\_rate=0.1}, and ensured full interpolation of the training data.

\paragraph{Empirical Results and Analysis}
\textcolor{red}{Figure \ref{fig:xgboost}} presents the empirical excess variance of the interpolating XGBoost model. In \textcolor{red}{Figure \ref{fig:xgboost} (a)}, the empirical variance of XGBoost drastically diverged from the $\mathcal{O}(n/p)$ noise diffusion typical of $\ell_2$ ridgeless interpolants. Instead of uniformly dissipating noise energy, the XGBoost risk trajectory maintained a significantly elevated and slowly decaying profile. To quantitatively evaluate this trend, we regressed the XGBoost empirical risk against the theoretical $1/\log(p/n)$ scaling coordinate derived from our fixed-dictionary analysis (\textcolor{red}{Figure \ref{fig:xgboost} (b)}). The regression yielded a strong positive correlation with $R^2 = 0.8393$. While the dynamic, data-dependent nature of tree generation introduces algorithmic complexities that deviate from exact fixed-dictionary Basis Pursuit, this robust correlation suggests a profound underlying connection. It implies that the greedy splitting heuristics of adaptive trees inherently perform a dual $\ell_\infty$-norm maximization analogous to greedy coordinate selection. Therefore, XGBoost empirically inherits the noise localization penalty characteristic of the $\ell_1$ geometry, fundamentally resisting the rapid variance dissipation required for benign overfitting.

\end{document}